\documentclass{article}

\usepackage{times}
\usepackage{graphicx} %
\usepackage[FIGBOTCAP]{subfigure}
\usepackage{algorithm}
\usepackage{algorithmic}
\usepackage{wrapfig}

\usepackage{color}
\usepackage{amsfonts}
\usepackage{amsmath}
\usepackage{amssymb}
\usepackage{bbm}
\usepackage{url}
\usepackage{amssymb}
\usepackage{amsbsy}
\usepackage{amsthm}
\usepackage{wrapfig}
\usepackage{hyperref}
\usepackage{footnote}
\usepackage{array}
\DeclareGraphicsExtensions{.jpeg,.png,.pdf} %

\usepackage{nips15submit_e,times}

\title{\cnqrname : Confusion Queried Online Bandit Learning}

\author{
Daniel Barsky \\
Department of Electrical Engineering\\
Technion Institute of Technology\\
Haifa, Israel 3200003 \\
\texttt{danielbr@tx.technion.ac.il} \\
\And
Koby Crammer \\
Department of Electrical Engineering\\
Technion Institute of Technology\\
Haifa, Israel 3200003 \\
\texttt{koby@ee.technion.ac.il} \\
}

\nipsfinalcopy %

\newtheorem{theorem}{Theorem}
\newtheorem{lemma}[theorem]{Lemma}

\newfont{\msym}{msbm10}
\newcommand{\reals}{\mathbb{R}}
\newcommand{\half}{\frac{1}{2}}

\newcommand{\sign}{{\rm sign}}
\newcommand{\paren}[1]{\left({#1}\right)}
\newcommand{\brackets}[1]{\left[{#1}\right]}
\newcommand{\braces}[1]{\left\{{#1}\right\}}

\newcommand{\abs}[1]{\left\vert{#1}\right\vert}

\newcommand{\Exp}[1]{{\rm E}\left[{#1}\right]}

\newcommand{\comdots}{, \ldots ,}

\newcommand{\beq}[1]{\begin{equation}\label{#1}}
\newcommand{\eeq}{\end{equation}}
\newcommand{\beqa}{\begin{eqnarray}}
\newcommand{\eeqa}{\end{eqnarray}}

\newcommand{\figref}[1]{Fig.~\ref{#1}}

\newcommand{\thmref}[1]{Theorem~\ref{#1}}

\newcommand{\appref}[1]{App.~\ref{#1}}
\newcommand{\lemref}[1]{Lemma~\ref{#1}}

\newcommand{\tabref}[1]{Table~\ref{#1}}
\newcommand{\tran}[1]{{#1}^{\top}}

\newcommand{\eqnref}[1]{Eqn.~(\ref{#1})}

\newcommand{\mb}[1]{{\boldsymbol{#1}}}

\newcommand{\vx}{\mathbf{x}}
\newcommand{\vxi}[1]{\vx_{#1}}
\newcommand{\vxin}[2]{\vx_{#1,#2}}

\newcommand{\vxii}{\vxi{t}}

\newcommand{\vphii}[1]{\mathbf{\phi}\left(#1\right)}

\newcommand{\vPhii}[1]{\mathbf{\Phi}\left(#1\right)}

\newcommand{\vm}{\mathbf{m}}
\newcommand{\vmi}[1]{\vm_{#1}}
\newcommand{\vmii}{\vmi{t}}
\newcommand{\vmis}{\vmi{t}^\star}

\newcommand{\vn}{\mathbf{n}}
\newcommand{\vni}[1]{\vn_{#1}}
\newcommand{\vnii}{\vni{t}}

\newcommand{\yi}[1]{y_{#1}}
\newcommand{\yii}{\yi{t}}

\newcommand{\indkt}{k_t}

\newcommand{\indlt}{l_t}

\newcommand{\indst}{m^\star_t}

\newcommand{\indm}[1]{m_{#1}}
\newcommand{\indmt}{m_t}
\newcommand{\indn}[1]{n_{#1}}
\newcommand{\indnt}{n_t}

\newcommand{\indi}[1]{i_{#1}}
\newcommand{\indit}{\indi{t}}
\newcommand{\indj}[1]{j_{#1}}
\newcommand{\indjt}{\indj{t}}

\newcommand{\vu}{\mathbf{u}}
\newcommand{\vut}{\tran{\vu}}

\newcommand{\vuti}[1]{\vut_{#1}}

\newcommand{\vw}{\mb{w}}
\newcommand{\vwi}[1]{\vw_{#1}}
\newcommand{\vwii}{\vwi{t}}

\newcommand{\vwt}{\tran{\vw}}
\newcommand{\va}{\mathbf{a}}

\newcommand{\tvw}{\tilde{\mb{w}}}

\newcommand{\tvwi}[1]{\tvw_{#1}}
\newcommand{\tvwti}[1]{\tran{\tvwi{#1}}}
\newcommand{\tvwii}{\tvwi{t}}
\newcommand{\tvwtii}{\tran{\tvwii}}

\newcommand{\vv}{\mb{v}}
\newcommand{\vvt}{\tran{\vv}}

\newcommand{\vb}{\mb{b}}

\newcommand{\ma}{\mathbf{A}}

\newcommand{\mi}{\mathbf{I}}

\newcommand{\newstufffroma}[1]{}

\newcommand{\newstufffrom}[1]{}

\newcommand{\oldnote}[2]{}

\newcommand{\commentout}[1]{}

\newcommand{\inner}[2]{\left< {#1} , {#2} \right>}

\newcounter {mySubCounter}
\newcommand {\twocoleqn}[4]{
  \setcounter {mySubCounter}{0} %
  \let\OldTheEquation \theequation %
  \renewcommand {\theequation }{\OldTheEquation \alph {mySubCounter}}%
  \noindent \hfill%
  \begin{minipage}{.40\textwidth}
\vspace{-0.6cm}
    \begin{equation}\refstepcounter{mySubCounter}
      #1
    \end {equation}
  \end {minipage}
~~~~~~
  \addtocounter {equation}{ -1}%
  \begin{minipage}{.40\textwidth}
\vspace{-0.6cm}
    \begin{equation}\refstepcounter{mySubCounter}
      #3
    \end{equation}
  \end{minipage}%
  \let\theequation\OldTheEquation
}

\newcommand{\vzero}{\mb{0}}

\newcommand{\algname}{{CNQR}}

\newcommand{\cnqrname}{{CONQUER}}

\newcommand{\vz}{\mb{z}}
\newcommand{\vzi}[1]{\vz_{#1}}

\newcommand{\vzii}{\vzi{t}}

\newcommand{\vzt}{\vz^\top}
\newcommand{\vzti}[1]{\vzt_{#1}}
\newcommand{\vztii}{\vzti{t}}

\newcommand{\deltatk}[1]{\Delta_{t,#1}}

\newcommand{\sigmai}[1]{\sigma_{#1}}
\newcommand{\sigmaii}{\sigmai{t}}
\newcommand{\hsigmai}[1]{\hat{\sigma}_{#1}}
\newcommand{\hsigmaii}{\hsigmai{t}}

\newcommand{\deltas}[1]{\Delta_{#1}^\star}
\newcommand{\deltats}{\deltas{t}}

\newcommand{\hdeltatk}[1]{\hat{\Delta}_{t,#1}}

\newcommand{\hdeltats}{\hat{\Delta}^\star_{t}}

\newcommand{\ai}[1]{A_{#1}}
\newcommand{\aii}{\ai{t}}
\newcommand{\atk}[2]{A_{#2,#1}}
\newcommand{\attk}[1]{\atk{#1}{t}}

\newcommand{\epst}[1]{\epsilon_{#1}}
\newcommand{\epstk}[1]{\epsilon_{t,#1}}

\begin{document}

\maketitle

\begin{abstract}
  We present a new recommendation setting for picking out two items
  from a given set to be highlighted to a user, based on contextual
  input. These two items are presented to a user who chooses one of
  them, possibly stochastically, with a bias that favors the item with the higher
  value. We propose a second-order algorithm framework that members of it use uses {\em relative}
  upper-confidence bounds to trade off exploration and
  exploitation, and some explore via sampling. We analyze one algorithm in this framework in an adversarial setting
  with only mild assumption on the data, and prove a regret bound of
  $O(Q_T + \sqrt{TQ_T\log T} + \sqrt{T}\log T)$, where $T$ is the number of rounds
  and $Q_T$ is the cumulative approximation error of item values using a linear
  model. Experiments with product reviews from $33$ domains show the
  advantage of our methods over algorithms
  designed for related settings, and that UCB based algorithms are inferior to greed or sampling based algorithms.
\end{abstract}

\section{Introduction}
Consider a book recommendation system triggered by users'
actions. Given a book search performed by the user, one can generate
heuristically about a dozen possible items to be presented,
by using author identity, other books in a series, and so on;
Yet, to be effective, only one or two results should be
highlighted to a user. The question is, given a set of possible books,
how to pick which two books should be presented to the user.

In this paper, we introduce confusion queried online bandit learning,
that given a set of items, picks two items to be presented to
the user. We assume the user will choose one of the two items
stochastically based on the items' values (or rewards) for the user.
Thus, if one item has clearly larger value than the other, the user
will likely choose it. However, if values of both items are similar,
the user will choose one at random. This relative
feedback is also used in the dueling bandit
setting~\cite{yue2012k,yue2011beat,yue2009interactively,shivaswamy2011online,shivaswamy2012online}.

We present \cnqrname , a second-order online algorithm framework designed for this
setting, and propose 7 algorithms within the framework. The framework assumes almost
nothing about the data generation process. One algorithm 
uses \emph{relative} upper confidence bounds, for which we show a regret bound
that is $O(Q_T + \sqrt{T(Q_T+\log T)\times \log T})$,
where $T$ is the number of rounds and $Q_T$ is the approximation error of
rewards using a linear model.

We evaluate the algorithms using product reviews from Amazon over $33$
product domains. On each iteration, a few (between 5 and 20)
possible items are chosen and presented to each algorithm, which picks
two of them based on their reviews. The algorithm then receives a
stochastic binary preference feedback, with a bias towards the item with
the higher star rating, between $1$ and $5$ stars. Algorithms were
evaluated on a test set. The \cnqrname\ algorithms show good performance on
these tasks, generating better results than two baselines designed for a related setting.

Online learning with relative feedback is also useful in many applications
where feedback is received from a human annotator. In such cases,
producing a full annotation of each instance can be time-consuming and
prone to mistakes. Producing a comparison between only two
options can be much quicker and easier. One such setting is the
sentence dependency parsing problem, where generating a full parse tree
is much harder than deciding between two possible sources for a single word
in a sentence. We also describe and analyze an algorithm for
learning to parse sentences in natural language using light feedback
regarding dependency parse trees~\cite{MejerCr10}. In this setting,
the algorithm can query for feedback on only a single
edge, while it's evaluated on the entire parse tree.

\section{Problem Setting}
Online learning is performed in rounds or iterations. On iteration
$t$, the algorithm receives a set of $K$
\footnote{The algorithm and
  analysis can be extended to the case where there are a different number of
  items on each round, that is $K=K(t)$. We use a fixed value
  $K$ for simplicity.}
  items. The algorithm then chooses (indices of) two items, $\indmt,
\indnt \in \braces{1 \comdots K} \triangleq \mathcal{K}$, and receives
a {\em stochastic} binary feedback, $\yii\in\braces{+1, -1}$,
indicating which of its two choices is better. A feedback of $\yii=+1$
indicates that item $\indmt$ has higher value or reward than item
$\indnt$, and vice-versa for $\yii=-1$. Additionally, the algorithm
receives reward for the two items it picked.

We assume that the reward function is bounded, $\vert
r\paren{t,m} \vert \leq 1 ~ \forall m\in\mathcal{K}$, and $t=1,2,\dots
T$. The reward value is not known or given to the algorithm, which is
only evaluated by it. The algorithm then proceeds to the next round.
Since we make almost no assumptions about the reward function, and it
is never revealed to the algorithm directly, our algorithm operates
in a fully adversarial setting - our only assumption is that
the reward function can be approximated using a linear function, for if
it's not the case, then there is no hope for any linear model based
algorithm. The reward in our setting is deterministic, and all stochasticity
in our model is assumed to be in the feedback process.

The feedback in our setting is assumed to be Bernoulli stochastic with a parameter
that is proportional to the difference between the rewards of our two
selected items, that is for $y\in\braces{\pm 1}$,
\begin{align}
    \Pr\paren{\yii = y} = %
    \frac{1+y\cdot\half\paren{r\paren{t, \indmt} - r\paren{t, \indnt}}}{2}  %
\end{align}
Note that under this assumption, if $r\paren{t, \indmt}$ and
$r\paren{t, \indnt}$ are close to each other, the feedback $\yii$ will be either $+1$
or $-1$ with equal probabilities, whereas if the rewards are
significantly different, the feedback will be biased towards the
item with the higher reward.

Another interpretation for this setting is the feedback being (mostly)
deterministic, but the {\em feedback provider} observing noisy rewards:
instead of observing $r\paren{t,\indmt}, r\paren{t,\indnt}$,
she is shown $\hat{r}\paren{t,\indmt}, \hat{r}\paren{t,\indnt}$, which are
given as follows:
\[
\hat{r}\paren{t,m} = \begin{cases}
+1  &   w.p. \frac{1+r\paren{t,m}}{2} \\
-1  &   w.p. \frac{1+r\paren{t,m}}{2}
\end{cases}
\]
The feedback provider returns $\yii = +1$ if
$\hat{r}\paren{t,\indmt} > \hat{r}\paren{t,\indnt}$,
$\yii = -1$ if $\hat{r}\paren{t,\indmt} < \hat{r}\paren{t,\indnt}$, and ties are broken
arbitrarily with equal probabilities. Regret is now calculated using
\emph{expected} noisy rewards. Proof of the equivalence appears in \appref{sec:noisy_reward}
in the supp. material.

Let $r_t$ denote the instantaneous regret at round $t$, which we
define as the difference between the reward of the best item,
and the reward the algorithm received:
\begin{align}
    r_t &= \underset{m\in\mathcal{K}}{\max} ~r\paren{s,m}  -
    \max\braces{ r\paren{s,\indm{s}} ,  r\paren{s,\indn{s}} }~.
\label{regret}
\end{align}
The goal of the algorithm is to minimize its cumulative regret,
$R_t = \sum_{s=1}^{t} r_s$. We denote the
optimal item by $\indst = \arg\max_{m \in \mathcal{K}} r\paren{t,m}
$. The regret is with respect to the better of the two items since we
assume a user is looking for a single item, as long it has very high
value - one very high-valued item and one very low-valued
item are a better combination that two mid-value items, even
though the latter might have higher average value.

We focus on linear models. For each item $\vxin{t}{m}$,
let $\vPhii{\vxin{t}{m}}=\vPhii{\vxii,m}\in\reals^{D}$ denote a feature vector
representing the m$th$ item in the set $\vxii$. We assume that the
feature vector is bounded and of unit Euclidean norm, that is,
$\Vert \vPhii{\vxii, m} \Vert = 1$, for all $t = 1, 2, \dots ~,~
m\in\mathcal{K}$. We wish to approximate the reward function
using a linear function, for some vector $\vu\in\reals^{D}$,
\(
  \vut \vPhii{\vxii, m} ~.
\)
We assume the dot product between the vector $\vu$ and the
feature vector of \emph{any} incoming instance is bounded by 1
in absolute value, meaning $\vert\vut\vPhii{\vxii, m} \vert \leq
1$ for all $m\in\mathcal{K},t$.

\section{Related Work}
We compare our framework to similar relevant methods in terms of regret formulation,
feedback, assumptions and more in \tabref{table_regret_summary} in the supp. material.
In addition, \tabref{table_summary} in the supp. material
shows our method side-by-side with similar contextual bandits methods with partial feedback.
In this section, we outline previous works similar or relevant to ours.

The dueling feedback model was initially explored in the bandits
setting, where several explore-then-exploit solutions were proposed ~\cite{yue2012k, yue2011beat}.
The regret in these settings is more general than ours,
since they do not assume any reward for a specific arm - only comparison
results between two arms, their regret being proportional to  the probability
that one arm will beat another in such a comparison. If we assume that each arm
has a reward associated with it, and that the probability of an arm defeating
another arm is proportional to the difference in rewards between the two arms,
then the \emph{weak regret} defined there is equivalent to our
regret formulation. The \emph{strong regret} used there
is stronger than ours.
An algorithm for interactively training information
retrieval was proposed for the online
learning as well \cite{yue2009interactively}.
Another similar setting is the online learning with
preference feedback
\cite{shivaswamy2011online,shivaswamy2012online}.
A comprehensive survey on bandit learning was recently published \cite{DBLP:journals/ftml/BubeckC12}.
All this line
differs from our, as we consider additional side information, or
context.

There has been a considerable amount of work on contextual online prediction
with light feedback. The Banditron algorithm~\cite{kakade2008efficient}
receives binary true-false feedback for a single prediction, and maintains a
$O\paren{T^\frac{2}{3}}$ regret bound.
The Banditron algorithm was extended in few ways~\cite{wang2010potential,valizadegan2011learning}. Confidit~\cite{DBLP:journals/ml/CrammerG13}
operates in a similar setting to Banditron, but assumes a specific stochastic label generating
mechanism, and achieves $O\paren{\sqrt{T} \log T}$ regret.  Our work
builds on this, but has few major differences: First, Confidit
uses binary (true-false) feedback on a single prediction, while
we use binary relative (dueling) feedback.
Second, they assume a specific stochastic label generation process, while we do not.
This assumption is expressed in the regret, where
Confidit assumes stochastic regret, whereas
we use a general and deterministic regret formulation.
Third, their analysis is on the expected zero-one
reward, while ours is a high-probability bound over any bounded
reward. Fourth, their feedback is deterministic, while ours is allowed
to be noisy (or stochastic). Finally, the multiclass categorization
problem they consider is a special case of our setting, the
reward of a single item being $1$, and the reward of the rest being $0$.
We stress that most if not all
previous works on UCB construct {\em absolute} confidence bounds
based on the estimated quantities, while we construct both absolute and {\em relative}
confidence bounds between two elements (see \eqref{eps}).

A similar linear model to
ours was used in a setting with true-false-don't know feedback (the KWIK
setting) on a single prediction~\cite{walsh2009exploring}.%
This work was extended and showed some improved results~\cite{ngo2013upper}.
Newtron~\cite{hazan2011newtron} is a second order descent
method for the online multiclass bandit setting, for which there is a bound on
the \emph{log-loss}, unlike the direct regret formulation we
used. A similar solution was used
for binary classification from single and multiple teachers~\cite{dekel2012selective}, but with
binary (true-false) feedback. A similar algorithm was also used in the
ordered ranking problem, with feedback indicating the intersection
between the predicted and optimal ordered results~\cite{gentile2012multilabel}.
Their regret formulation is also general, yet it takes into
account the number of labels output (which is constant in our case), and the order
in which the labels were output, which we do not care about. In addition,
their algorithm requires several parameters (as opposed to our single parameter),
and makes several assumptions which we do not.

Using confidence bounds to trade off between
exploration and exploitation has been proposed~\cite{auer2003using} and later
extended~\cite{dani2008stochastic}. However, their work uses a
confidence bound in a single prediction to evaluate the accuracy of a
prediction, while we use similar methods to evaluate the possible
confusion between two possible predictions.
Multiclass online learning with bandits feedback can be considered a
bandit problem with side information. An epoch-greedy algorithm for
contextual multi-armed bandits has been introduced~\cite{langford2007epoch},
with a regret bound of $O\paren{\sqrt{T}\log T}$.

\section{Confusion Queried Online Bandit Learning}
We now describe an online, second order algorithm framework for the
recommendation setting described above. This framework will define
a group of algorithms, which differ only in their method of selecting
the two items to be displayed $\indmt, \indnt$, and are identical otherwise.
\begin{wrapfigure}{r}{0.49\textwidth}
	\vspace{-0.3cm}
	\begin{minipage}{0.49\textwidth}
		\begin{algorithm}[H]
			\caption{\cnqrname\ - Framework Outline}
			\label{conquer_general}
			\begin{algorithmic}[1]
				\REQUIRE $\delta \in \paren{0, 1}$ (used in \eqref{eps} and \eqref{eta})
				\STATE Initialize $\vwi{0} = \vzero \in \reals^{D}, \ai{0} = \mi_{D\times D}$
				\FOR {$t = 1$ to $T$}
				\STATE Receive set of items $\vxii \in \mathcal{X}^K$
				\STATE Project $\tvwi{t-1} = \arg\min~d_t\paren{\vw, \vwi{t-1}}$ \\
				Subject to $\vert\vwt\vPhii{\vxii, m} \vert \leq 1 ~ \forall m \in \mathcal{K}$
                \STATE Set $\indmt = \arg\max_{m\in\mathcal{K}}\hdeltatk{m}$
                \STATE Set $\indnt$ according to \tabref{conquer_algs}
				\STATE Output $\indmt, \indnt$
				\STATE Receive feedback $\yii \in \braces{ \pm 1 }$
				\STATE Set $\vzii = \half \yii \cdot \paren{\vPhii{\vxii, \indmt} - \vPhii{\vxii, \indnt}}$
				\STATE Update $\aii = \ai{t-1} + \vzii \vztii$
				\STATE Update $\vwii = \aii^{-1}\paren{\ai{t-1} \tvwi{t-1} + \vzii}$
				\ENDFOR
				\ENSURE $\tvwi{T}$
			\end{algorithmic}
		\end{algorithm}
	\end{minipage}
\end{wrapfigure}
Every algorithm in our framework maintains a
linear model, $\vwii\in\reals^{D}$, which it uses to estimate the
reward of each item.
In addition, our algorithms maintain a
positive definite (PD) matrix, $\aii \in \reals^{D\times D}$, used sometimes to
estimate the confidence in the score of an item or the confusion between two possible items
from a given set of items $\vxii$. We update $\vwii$ and $\aii$ using
the standard second-order update rule, using the vector
$\vzii = \half \yii \cdot \paren{\vPhii{\vxii,\indmt} - \vPhii{\vxii, \indnt}}$
as the update vector, where $\yii\in\braces{\pm 1}$ is the feedback
received by the algorithm. Let us define the following Mahalanobis distance between
two vectors $\va, \vb\in\reals^{D}$, $d_{t}\paren{\va,\vb}=\half\paren{\va-\vb}^\top\aii \paren{\va-\vb}$,
and let $\tvwi{t-1}\in\reals^{D}$ denote an orthogonal projection of
$\vwi{t-1}$ w.r.t. the above-mentioned distance that satisfies
$\vert \tvwi{t-1} \vPhii{\vxii, m} \vert \leq 1~\forall m \in \mathcal{K}$.
This projection stage is required for our regret analysis. Let
\(
    \hdeltatk{m} = \tvwti{t-1} \vPhii{\vxii, m}~,
\)
denote the score of item $m$ under our model. Since the reward at round $t$
is the maximum of the rewards of our two predictions, we are OK if we output
$\indst = \arg\max_{m \in \braces{1 \comdots K}} \deltatk{m}$
as one of our predictions.

Let $\epstk{m}$ denote the amount of confidence we have in the score
of item $m$ according to our model,
\(
    \epstk{m}^2 \!=\! \eta_t \!\times\! \tran{ \vPhii{\vxii, m}} \!\ai{t-1}^{-1}\! \vPhii{\vxii, m}.
\)
A higher value of $\epstk{m}$ implies less confidence in the value of $\hdeltatk{m}$.
Let
$\epstk{m, n}$ denote the amount of our confusion between items $m,
n \in \braces{1 .. K}$ in the set $\vxii$:
\begin{align}
    \epstk{m, n}^2 &= \eta_t \times \tran{ \paren{ \vPhii{\vxii, m} -
        \vPhii{\vxii, n}} } \ai{t-1}^{-1} \paren{ \vPhii{\vxii, m} - \vPhii{\vxii, n}} \label{eps} \\
        \eta_t &= d_{0}\paren{\vu, \vzero} \!+\! 2\sum_{s=1}^{t} q_s \!+\! 2\sum_{s=1}^{t-1} \vzti{s}\ai{s-1}^{-1}\vzi{s} \!+\!36\ln\frac{t+4}{\delta} \label{eta}
\end{align}
Here, $q_s$ is the approximation error of the reward function using a
linear model at time $t$, and $\delta$ is a parameter of the algorithm,
indicating with what probability the cumulative regret is bounded by the
regret bound. This form of $\eta_t$ is required only
for the analysis, and is not practical. In practice, we set
$\eta_t = \eta$ to be a constant value.

On iteration $t$, each of our algorithms receives a set $\vxii$, and starts by projecting
its current model vector, $\vwi{t-1}$, onto the set mentioned above.
The algorithm then pick two items to be displayed. Once the two items $\indmt, \indnt$
are selected by the algorithm,
it receives the relative stochastic binary feedback $\yii \in \braces{ \pm 1 }$
and performs a standard second order update using the difference vector
$\vzii$. We call the algorithm framework {\em \cnqrname} for CONfusion QUERied online bandit
learning. Below we also use \algname\ for a shorter abbreviation. Algorithm~\ref{conquer_general} describes our framework formally. In addition,
\tabref{conquer_algs} shows the method of selecting $\indnt$ in 3 algorithms
we propose within our framework.

We will provide regret analysis for the \algname -GNC  algorithm.
In this algorithm, the first item $\indmt$ is chosen as the best item according to the model at time $t$, $\indmt
= \arg\max_{m \in \mathcal{K}} \hdeltatk{m}$, and the second item is the
"optimistically best" item (that's different from $\indmt$). In other
words, we define relative upper-confidence scores,
$\beta\paren{n} = \hdeltatk{n} - \hdeltatk{\indmt} + \epstk{n, \indmt}$.

Next, we output the best item in terms of
$\beta\paren{n}$ that is different from our first choice, $\indnt
= \arg\max_{n \in \mathcal{K}/\braces{\indmt}} \beta\paren{n}$. We
note that it is a mixture of first-order (as $\indmt$ is just the
best according to the model) and upper confidence bound (as
$\indnt$ is chosen using a confidence bound)
strategies. Additionally, these confidence bounds are {\em relative}
to the best action $\indmt$, and not absolute as done
in general.
Finally, we allow \algname -GNC {\em not} to pick a
second item, if it is certain that the first item picked is the
best one, even in face of uncertainty, that is if
$\beta(\indnt)<0$.  Conceptually, if the difference between the
reward of the two item is higher than the confidence interval,
then we are sure the first item is better, and the second one
can be ignored.  In this case no update will be performed, as
there is no relative feedback.

\paragraph{Computing the Projection: }

The algorithm requires computing an orthogonal projection problem,
which means minimizing a quadratic function subject to (multiple)
linear constraints,
\(
    \tvwi{t-1}  \!=\! {\arg\min}_{\braces{\vw | \vert\vwt\vPhii{\vxii, m} \vert \leq 1 ~ \forall m \in \mathcal{K}}}~d_t\paren{\vw, \vwi{t-1}}~.
\)
This problem is convex, yet we could not find a closed form solution.
We propose to solve it using a sequential algorithm which projects on
a single constraint~\cite{CensorZe97}. A detailed description of this
method can be found in~\appref{sec:sup_dependency_parsing}.

\begin{wrapfigure}{r}{0.55\textwidth}
\vspace{-0.4cm}
\begin{minipage}{0.55\textwidth}
\begin{center}
    \begin{tabular}{| >{\small}l | >{\small}l | } \hline
    \textbf{Name}               & \textbf{$\indnt$ Selection}                                                           \\ \hline\hline
    \textit{Top Two Greedy}     & Second best in terms of score                                                         \\
    \textit{(\algname  -TTG)}   & $\arg\max_{n\in\mathcal{K}/\braces{\indmt}}\hdeltatk{n}$                              \\ \hline
    \textit{Greedy + Random}    & Random item from the remaining set                                                    \\
    \textit{(\algname  -GNR)}   & $\text{Random}\braces{\mathcal{K}/\braces{\indmt}}$                                   \\ \hline
    \textit{Greedy + UCB}       & Best in terms of UCB                                                                  \\
    \textit{(\algname -GNU)}    & $\arg\max_{n\in\mathcal{K}/\braces{\indmt}}\hdeltatk{n}+\epstk{n}$                    \\ \hline
    \textit{Greedy +}           & Best in terms of \textit{relative} UCB                                                \\
    \textit{Confusion}          & $\arg\max_{n\in\mathcal{K}/\braces{\indmt}}\hdeltatk{n}+\epstk{\indmt, n}$            \\
    \textit{(\algname -GNC)}    & (doesn't query or update if                                                           \\
    ~                           & $\hdeltatk{\indmt}-\hdeltatk{\indnt} > \epstk{\indmt, \indnt}$)                       \\ \hline
    \end{tabular}
    \caption{Algorithms within the \cnqrname\  framework}
    \label{conquer_algs}
\end{center}
\end{minipage}
\end{wrapfigure}

\paragraph{Time and Space: } We conclude this section by noting that
each step of the algorithm requires $O(D^2 + DK)$, where $O(D^2)$ is
needed to compute the inverse (and the projection) and the update of
the matrix $\aii$, and $O(DK)$ to compute the output and update the
parameters $\vwii$.  Note also the algorithm can also be run in dual
variables (i.e., in a RKHS). This has a twofold implication: (a) The
resulting reward model (2) can be made highly nonlinear in the
features, and (b) the running time per round can be made quadratic in
the number of rounds so far.

In the experiments reported below, we used a version of the algorithm
that maintains and manipulates a diagonal matrix $\ma$ instead of a
full one. All the steps of the algorithm remain the same, except the
update $\aii = \ai{t-1} + \vzii \vztii$ of line 17, that is replaced
with updating only the diagonal elements, $(\aii)_{r,r} =
(\ai{t-1})_{r,r} + (\vzii)_{r}^2$. The running time is reduced now
to $O(KD)$ (the projection depends also on the number of
iterations). The space needed is now $D$ for both $\vw$ and $\ma$.

\section{Regret Analysis for \algname -GNC}
We have no assumption on the data generation process, that is, how
items $\vxii$ were generated, and how the reward function $r\paren{t,m}$ is
defined. Yet, we are approximating the latter with a
linear function of the former (or its features). If the reward
function is far from being linear, then there is no hope for the
algorithm to work well. We thus quantify the approximation error, and
define
\(
    q_t = \underset{m\in\mathcal{K}}{\max} \left| r\paren{t,m} - \vut \vPhii{\vxii, m} \right|~,
\)
and let $Q_t = \sum_{s=1}^{t} q_s$ denote the cumulative approximation
error. In the analysis below, we compare the performance of our algorithm to
the performance of any linear model $\vu$. We therefore denote by
\(%
    \deltatk{m}=\vut\vPhii{\vxii, m},
\) %
the approximate reward of item $\vxi{t,m}$. Thus, the label maximizing the approximate reward is,
\(%
    \indst=\arg\max_{m\in\mathcal{K}}~\vut\vPhii{\vxii, m},~
\)%
with associated approximate reward
\(%
    \deltats=\vut\vPhii{\vxii, \indst} ~.
\)%

We now compute a bound on the regret $R_t$ defined in
\eqref{regret}. The regret is a sum of two terms: approximation error
(due to approximating the reward function with a linear model) and
estimation regret (due to stochastic feedback and online learning).
The approximation error measures how well our (linear) model is, and
is not algorithm dependent.

First, we bound the instantaneous regret by a sum of the
approximation error $q_t$ and the confusion coefficient for that round, under the assumption that
$\left|\paren{\hdeltatk{\indnt}-\hdeltatk{\indmt}}-\paren{\deltatk{\indnt}-\deltatk{\indmt}}\right|$
is bounded by the confusion coefficient, for any $t, \indmt$ and
$\indnt$ (\lemref{lemma_r_t_q_t}). Then, we   show that this
assumption holds with high probability (\lemref{lemma_(wtx-utx)2},
\lemref{lemma_E(wtx-utx)_q_t}, \lemref{lemma_delta_deltahat}; all
given in the \appref{app:lemmas} in the supplementary
material). Finally, we bound in \thmref{thm_nonlinear} the
cumulative regret by $O\paren{Q_T + \sqrt{T+Q_T}
  \log T}$. The last term is small if the approximation error $Q_T$ is
small. Our proof builds on results developed of \cite{DBLP:journals/ml/CrammerG13}.
The main result of this section is \thmref{thm_nonlinear}, which its proof appears in \appref{app:lemmas} and \appref{app:proof_thm_nonlinear} in the
supp. material.
\begin{theorem}\label{thm_nonlinear}
In the setting described so far, the cumulative regret $R_t$ of \algname -GNC satisfy
\(
    R_t = 2Q_T +  \sqrt{2T}\paren{\sqrt{\paren{(2Q_T + A)
    B}} + B},
\)
with probability at least $1-\delta$ uniformly over the time horizon $T$,
where $A = d_0\paren{\vu,\vzero}+36\ln\frac{T+4}{\delta}$ and $B = 2DK\ln\paren{1+\frac{T}{dK}}$.

\end{theorem}

\begin{wrapfigure}{r}{0.5\textwidth}
	\vspace{-1.5cm}
	\begin{minipage}{0.5\textwidth}
		\begin{center}
			\includegraphics[width=1\columnwidth]{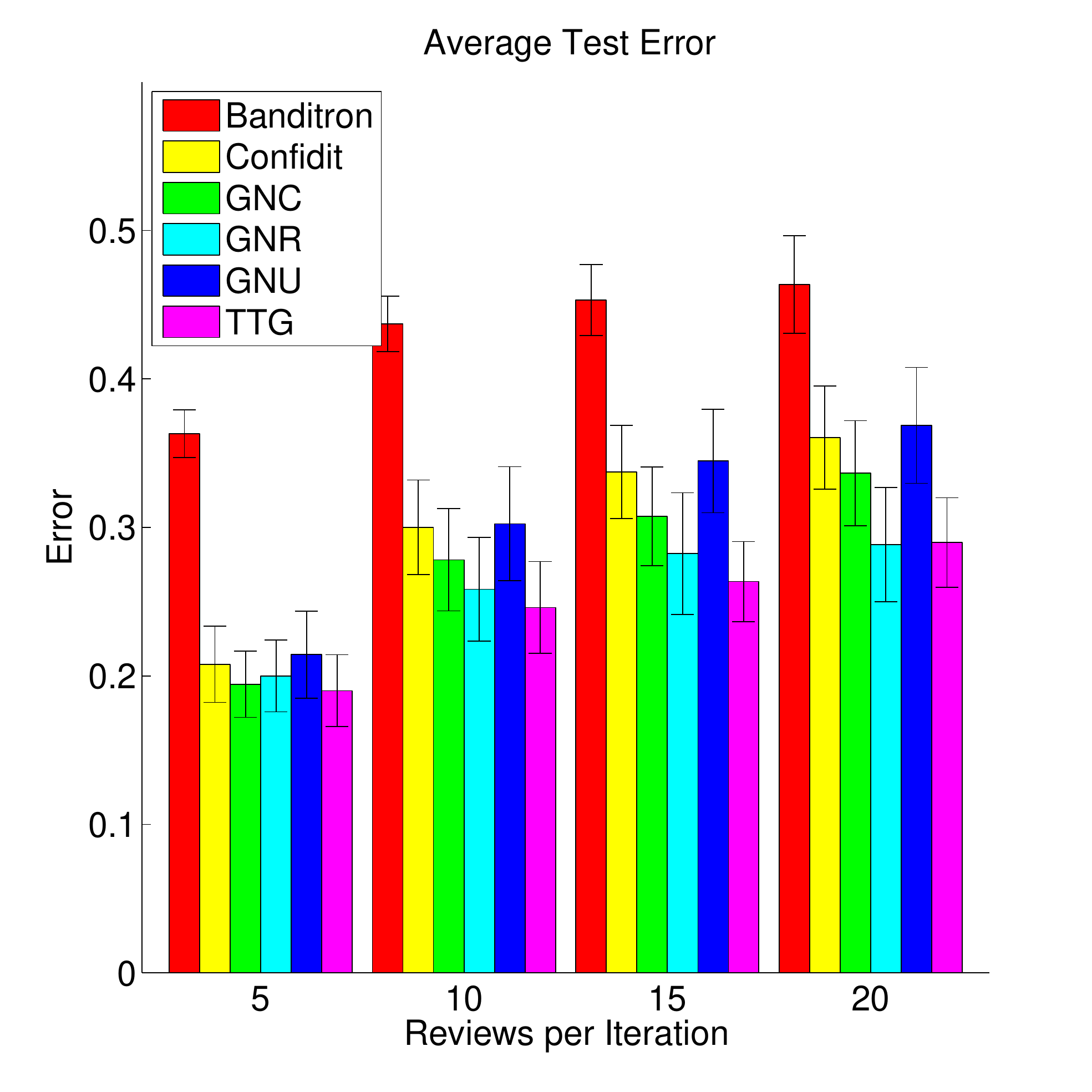}
		\end{center}
		\vspace{-0.8cm}
		\caption{Average error over all Amazon domains}
		\label{fig_amazon_avg}
	\end{minipage}
	\begin{minipage}{0.5\textwidth}
		\begin{center}
			\includegraphics[width=1.2\columnwidth]{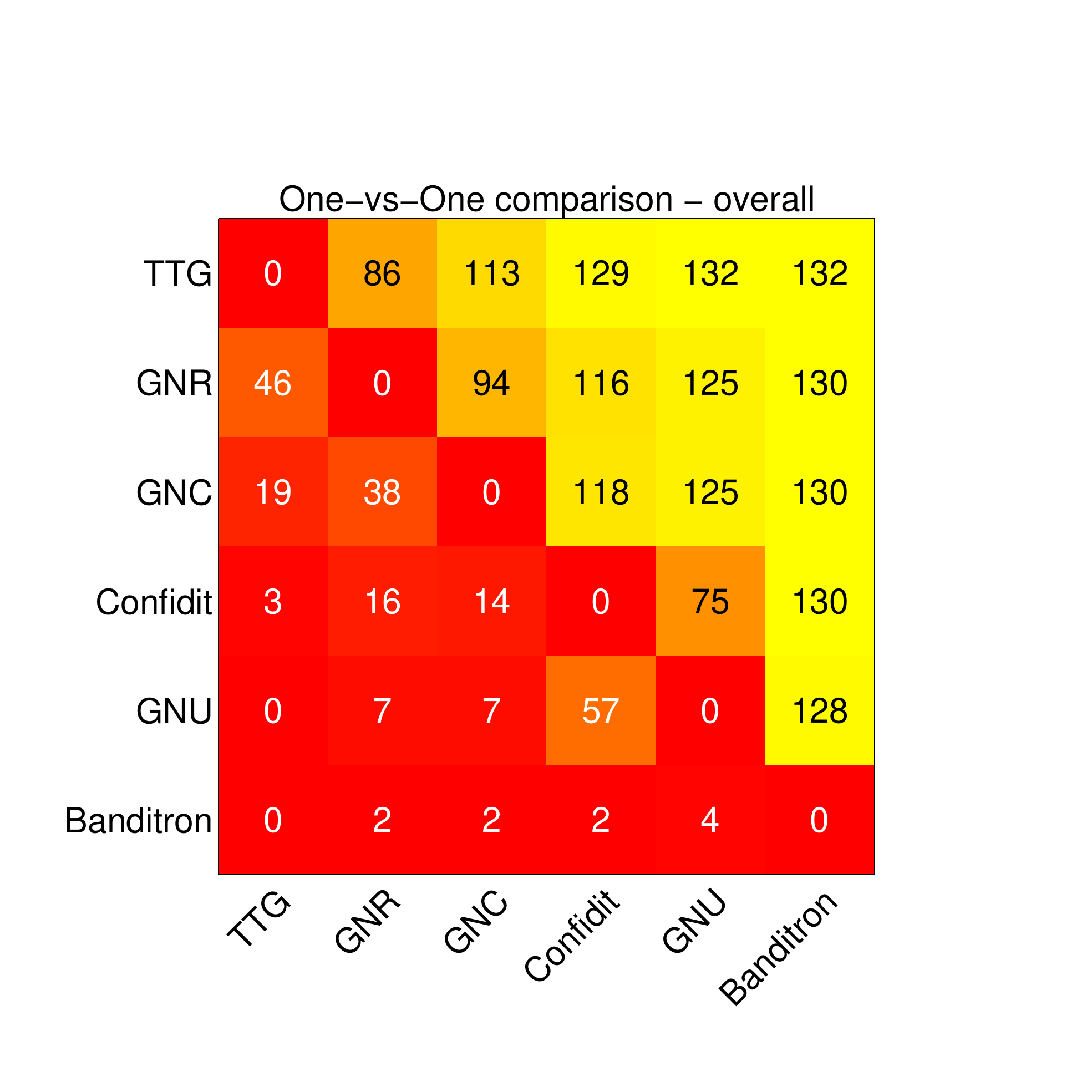}
		\end{center}
				\vspace{-0.8cm}
		\caption{One-vs-one competition of all algorithms over all experiments}
		\label{fig_amazon_comp}
	\end{minipage}
\end{wrapfigure}
\section{Experimental Study}

We evaluated our algorithms using the following recommendation setting
on reviews from Amazon. We used $341,400$ reviews on products from $33$
domains\footnote{android, arts, automotive, baby products,
  beauty, books, camera, cell-phones, clothing, computers,
  electronics, food, gardening \& pets, health, home, industrial,
  jewellery, kindle, kitchen, magazines, movies \& TV, MP3, music,
  musical instruments, office, patio, shoes, software, sports, toys,
  video games, videos, watches}. The reviews were preprocessed by
converting upper-case text to lower-case, replacing common non-word
patterns (such as common emoticons, 3 dots, links) with a unique mark,
removing HTML tags, and expanding abbreviations. We extracted bi-gram
features, with $6,255,811$ features per item. Each review
also comes with a rating, between $1$ and $5$ stars, which was
used as its (normalized) reward.

On each iteration, each algorithm is given a set of $K$ items from a
specific domain, each represented by a review written on that item.
The algorithm then picks two items $\indmt,\indnt$ based on these reviews, and
receives a stochastic bit $y\in\{\pm 1\}$ with probability
proportional to the difference in the number of stars each of the items
received. If the number of stars both reviews got are close, the bit will be $y\!=\!+1$ with
probability $\approx 0.5$.
Where the difference
is maximal ($4$), the bias is $(1 \!+\! 0.25 \!\times\! (5-1))/2 \!=\! 1$. This
process simulates the case where the two items are shown to a user,
who clicks on one item based on its relative value or reward.
Even though full feedback is available here,
relative feedback
simulates better the process of choosing one out of several suggested
products %
occurring in practice.

The error at time $t$ was defined as the difference between the
maximal rating in the reviews received at time $t$ and the rating of
the first review selected by the algorithm ($\indmt$), divided by $4$.
The highest value is $1$, and choosing a random review as the best
yields an average error of $0.5$. We experimented with sets of size
$K=5, 10, 15, 20$. In total, we have $33$ (domains) $\times$ $4$
(values of $K$) $=132$ experiments. Each trial was performed 10
times, and the results were averaged. The error bars appearing in
the figures indicate the $95$ percentile.

The reviews in each domain were divided into three sets: $15\%$ was
used as a development set to calibrate the parameter of each algorithm,
$10\%$ was used as a test set, and the remaining $75\%$ were
used for training. All algorithms had their parameters tuned on the
development set, over a grid of fixed size, separately per domain.

Once parameters were tuned, we executed all algorithms on the training set
(single iteration) and evaluated the resulting model on the test set.
We evaluated a total of six algorithms. In addition to the 
\cnqrname\ algorithms we outlined in~\tabref{conquer_algs}, we also evaluated two multiclass contextual bandits
algorithms - Confidit~\cite{DBLP:journals/ml/CrammerG13} and Banditron~\cite{kakade2008efficient} -
each picks a single item, and receives a binary feedback stating whether the picked item's
rating is maximal.

All algorithms that maintain second order
information - the matrix $\aii$ - were executed with a diagonal
matrix. Confidit and the \cnqrname\ algorithms require a scalar quantity
$\eta_t$ to be computed based of unknown quantities, such as the vector
$\vu$. Instead, we follow previous practice~\cite{DBLP:journals/ml/CrammerG13},
and set $\eta_t=\eta$ to a fixed value tuned on the development set.
Finally, since even now the projection is very time-consuming,
and since it is required only for the sake of analysis,
and was not improving performance in practice on small scale problems,
it was not performed in the experiments below, where the data dimension is high.

\figref{fig_amazon_avg} shows average test error over all
domains. The four blocks correspond to the number of possible items
per iteration $K$, and the eight bars per block correspond to the eight
algorithms run. \figref{fig_amazon_comp} shows a one-vs-one competition of all
algorithms one against the other in all trials - a value of $k$ in row $i$, column
$j$ indicates that algorithm $i$ did better than algorithm $j$ in $k$ trials.
These results are consistent with each other, meaning that algorithms that do
well on average also beat less successful algorithms in a majority of the trials.
Finally, \figref{fig_amazon} shows the test error for Automotive,
Books, and Kindle domains. These results show a similar trend to the average
results. Additional results are in \appref{app:results} in the supp. material.

Clearly, Banditron performs worst with error greater than
$0.3$, \algname  -TTG is the clear winner in both average error and trial wins, with
\algname  -GNR in close second place. \algname  -GNC is third, and Confidit and \algname -GNU follow.
We also experimented with algorithms employing a non-relative UCB policy to select $\indmt$,
but these yielded poor results. Few comments are in order. First, \emph{relative} UCB indeed outperforms
absolute UCB in the setting described above, providing justification for this choice, and our choice of analysis.
Second, it seems that greed or random exploration are better-suited to this setting than a UCB policy, both in the absolute
and relative variants. Third, the performance of all
algorithms relying on UCB policies deteriorates much faster as the number of
items per round increases compared to non-UCB algorithms - we hypothesize that
this is because as the number of items per round increases, the chance of choosing
two items with identical reward grows, and for these pairs of items the UCB factor
is more significant, whereas algorithms that choose based on score alone don't
suffer as much from the added confusion. 
Fourth, we evaluated additional algorithmic variants where the first choice of the algorithm is based on UCB, and the second is either second best UCB, or a random choice. Both these variants performed poorly compared to the variants we tried, and thus are omitted.

Finally, it is evident that avoiding making a query in some cases doesn't seem to hurt performance
significantly. However, we must point out that \algname  -GNC
avoided making queries on a maximum of $8.27\%$ of the rounds per domain, with common
values ranging around $0.5\% - 1\%$ (the full statistics of optional querying are in
\tabref{tab:optional_query} in the supp. material). This is due to the relatively small number of rounds
per domain, and we expect this number to increase if multiple iterations over the training data
are performed or the trial is run on a larger dataset.

\section{Light feedback in Dependency Parsing}
\label{sec:dep_parse}

\begin{figure*}[t]
	\centering
	\includegraphics[width=0.34\textwidth]{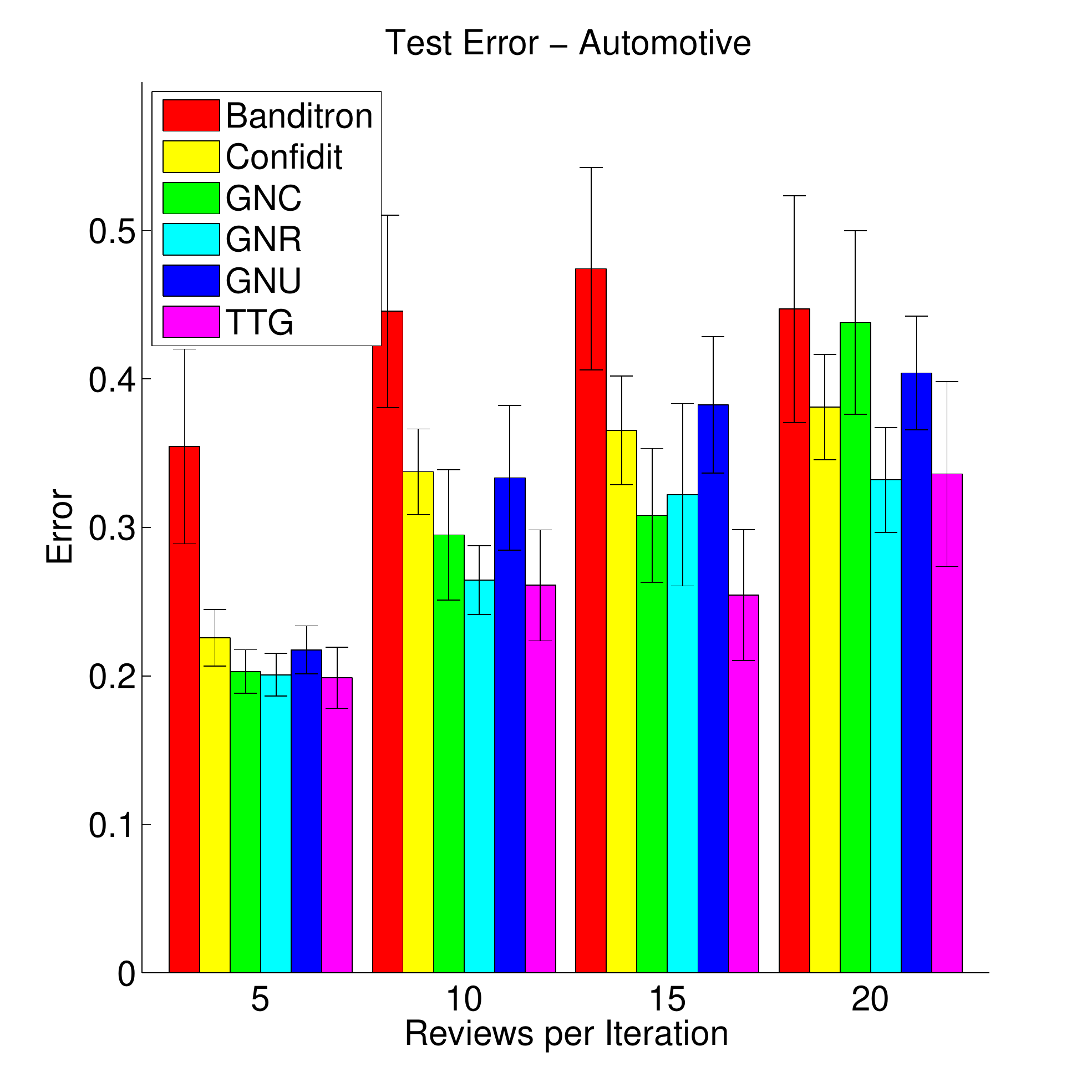}
		\hspace{-0.5cm}
	\includegraphics[width=0.34\textwidth]{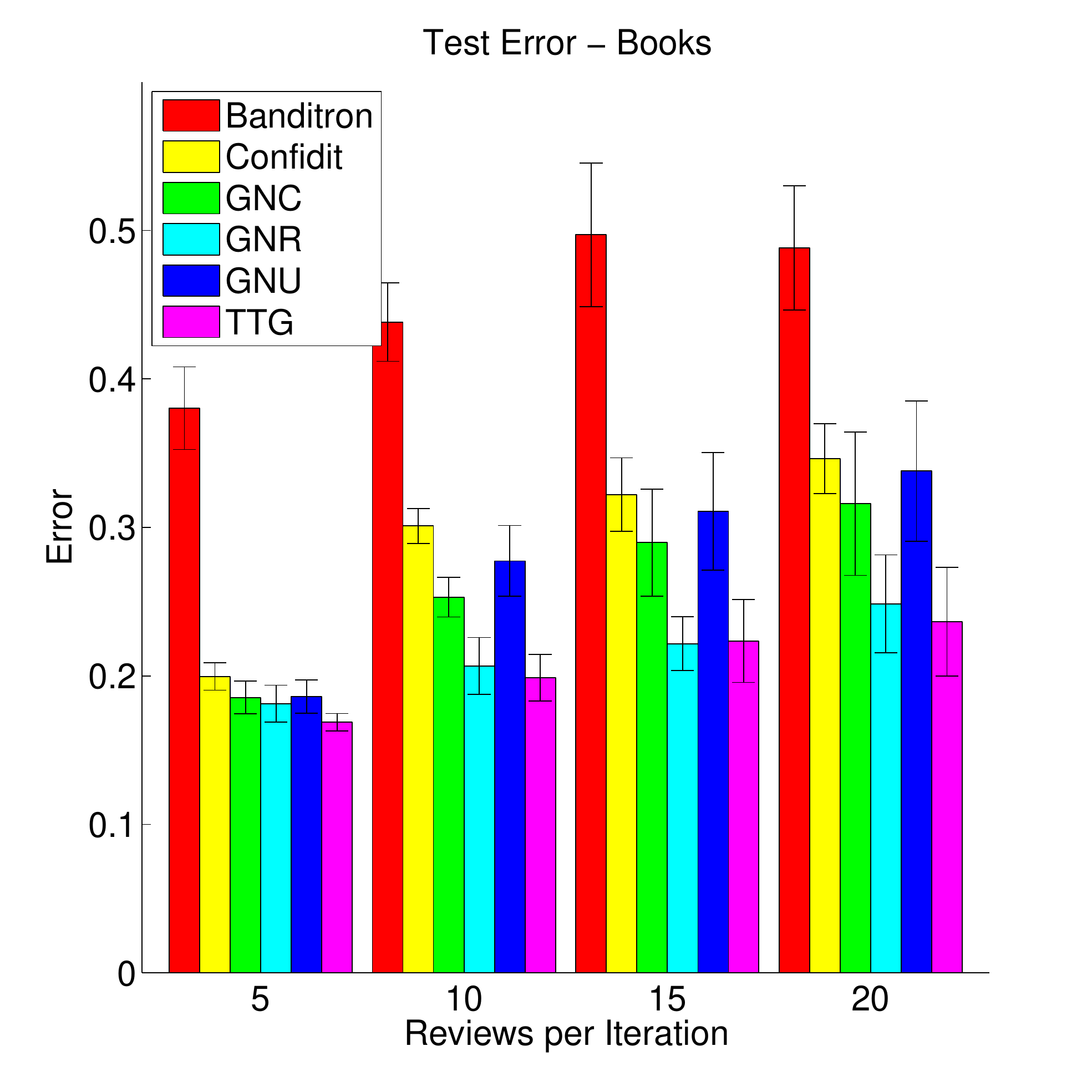}
	\hspace{-0.5cm}
	\includegraphics[width=0.34\textwidth]{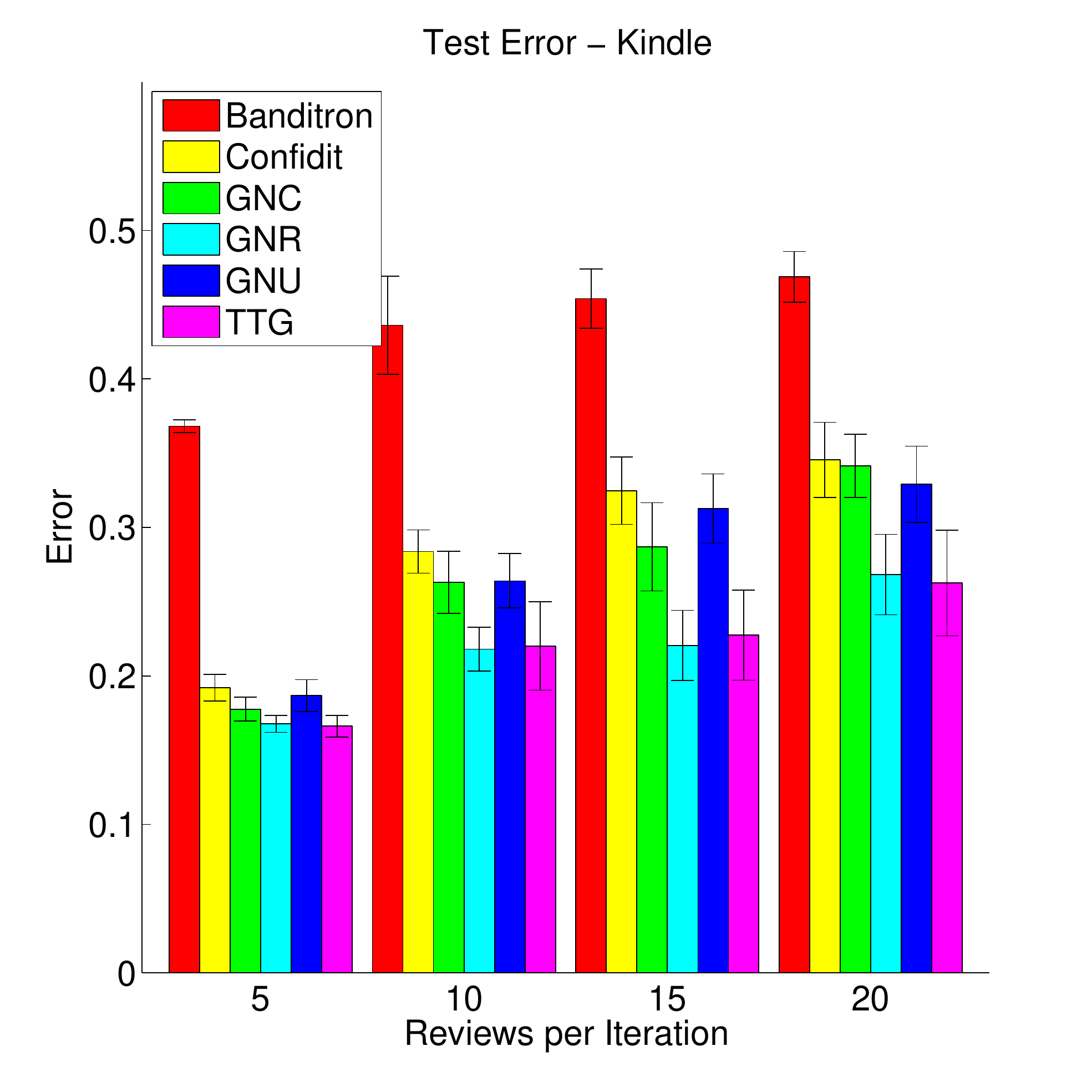}
	\caption{Average test error over Automotive, Books and Kindle domains.}
	\label{fig_amazon}
\end{figure*}

A setting related to ours is learning dependency parse trees with
light feedback~\cite{MejerCr10}, where items are sentences and are associated with full parse
trees. These trees are often generated by human annotators, in a process that is complex, slow and prone to mistakes, as for each
sentence a full correct feedback is required. We focus on partial feedback, where given a sentence, the algorithm outputs a {\em complete} parse tree,
but can request feedback on only a single word in the sentence, and feedback is only
required about the relative correctness of two alternative edges related to that word.

More formally, an item received at time $t$ is a sentence $\vxii$
of length $s_t \leq S$ in some natural language. The goal of the
algorithm is to generate a dependency parse tree $\vmii$ between the
words in the sentence. We assume that features about a tree decompose
to features about edges, and specifically the syntactic relation
between the i$th$ word and the j$th$ word can be captured in a feature
vector $\vphii{\vxii, i, j}\in\reals^D$, and the score of an edge
between words $i$ and $j$ is given by
$\tvwti{t-1}\vphii{\vxii, i, j}$. Given a possible dependency parsing
$\vm$ over $\vxii$, we construct an aggregated feature vector,
\(
    \vPhii{\vxii, \vm} = \frac{1}{s_t-1}\sum_{\paren{i, j} \in \vm}
    \vphii{\vxii, i, j} ~.
\)
The $\frac{1}{s_t-1}$ factor is required to ensure that $\vPhii{\vxii, \vm}$
is of unit norm. Given a model weight vector $\tvwi{t-1}$, the predicted parse
tree is computed by
\(
    \vmii = \arg\max_{\vm\in\mathcal{K}_t} \tvwti{t-1} \vPhii{\vxii, \vm}~.
 \)

 The confusion (or dueling) feedback in this setting is
 realized in the feedback given to the algorithm. Given some parse
 tree, e.g. the parse tree $\vmii$ maximizing the score, the algorithm
 selects some word index $i \in \{ 1 ... s_t \}$ and chooses an
 alternative source for it. The feedback it asks for and receives can
 be summarized in the question: ``Which of the $j$th and $k$th words of the
 sentence are a better source for the $i$th word?''.  For example in
 the sentence ``{\em I saw the dog with the telescope}'' the algorithm may
 ask ``{\em Was the dog seen by a telescope, or was the dog seen carrying a
   telescope?}''. While in this case, both answers may be valid both
 syntactically and semantically, it is not the case when we replace
 the word {\em telescope} with the word {\em bone}. The algorithm then
 uses this feedback to update its model. This type of
 focused relative feedback makes the work of a human annotator much
 easier - instead of parsing an entire sentence, she is only required
 to answer to a simple yes/no question.

 Our solution to this setting, \algname  -DP (\cnqrname\  Dependency Parsing), is outlined in
 Algorithm~\ref{conquer_dependency}, and \thmref{thm_parse} shows for it
 a regret bound of $O\paren{S \cdot\sqrt{T} \log T}$ with respect to the edge-accuracy evaluation
 measure, often used in dependency parsing. Due to lack of space, both
 algorithm and analysis appear in \appref{sec:sup_dependency_parsing}
 in the supp. material. A similar previous solution~\cite{MejerCr10}
 picks the best two possible alternatives  greedily, while we pick one
 greedily and the other  optimistically. Finally, we prove a regret bound,
 while they do not.

\section{Conclusions and Future Work}
We have introduced confusion queried online bandit learning, which
shares properties with both contextual bandits and dueling
bandits. We described a new algorithm framework for this task, suggested
few algorithms in it, and analyzed the regret for one of them
under very mild assumptions. We showed that the proposed online
second order algorithms are efficient. Extensive experiments we performed
with a recommendation system based on reviews showed the usefulness of
the setting and our framework. Two main insights are: (1) in this settings it seems that greedy or random exploration outperform UCB based exploration (2) Relative UCB (as we analyzed) outperform absolute UCB. Thus, exploration should be tuned appropriately when a relative preference choice is performed. 

Future work might include bounding the number of queries the algorithm
makes. We also plan to explore similar frameworks where more than two labels can
be presented for feedback, which would pick the most relevant subset from the set selected
by the algorithm. Last, but not least, we plan to experiment with our
dependency parsing solution, and apply our methods to other domains.

\small {
    \bibliographystyle{unsrt}
    \bibliography{bib_nips}
}

\newpage
\appendix
\section{Supplementary Material}

\subsection{Lemmas} \label{sec:lemmas}
\begin{lemma} \label{lemma_r_t_q_t}
If at time $t$, the algorithm maintains $\left|\paren{\hdeltatk{n}-\hdeltatk{m}}-\paren{\deltatk{n}-\deltatk{m}}\right|\leq \epstk{n,m}$ for all $m,n \in\mathcal{K}$, then
\(
    r_{t}=\underset{m\in\mathcal{K}}{\max}~ r(t,m) - \max \braces{r(t,\indmt), r(t,\indnt)} \leq 2q_t + 2\epst{t}~,
\)
where
\(
    \epst{t}^2 = 2\vztii\ai{t-1}^{-1}\vzii\times\eta_t
\)
for all $t = 1 , 2 \comdots T $.
\end{lemma}
\begin{proof}
Notice that from the description of the algorithm,
$\hdeltatk{n} - \hdeltatk{\indmt} \leq 0$ for any $n \in \mathcal{K}$.
In addition, notice that $\epstk{m,n} = \epstk{n,m}$ for any $m,n \in \mathcal{K}$.
From the algorithm description, we have,
\begin{align*}
    r_{t} &=& \underset{m\in\mathcal{K}}{\max}~ r(t,m) - \max \braces{r(t,\indmt), r(t,\indnt)} &\leq& \deltats- \max\braces{\deltatk{\indmt}, \deltatk{\indnt}} + 2q_t\\
    &\leq&\deltats- \deltatk{\indmt} + 2q_t &\leq& \hdeltatk{\indst}-\hdeltatk{\indmt} + 2\epstk{\indst, \indmt} + 2q_t \\
    &\leq& \hdeltatk{\indnt}-\hdeltatk{\indmt} + 2\epstk{\indnt, \indmt} + 2q_t &\leq&  2\epst{t} + 2q_t ~.
\end{align*}
Note that if $\beta\paren{\indnt} < 0$, the linear approximation term
is bounded by a negative number, and since it is non-negative by definition,
this means that the linear term is zero, and the regret is composed of the
approximation error term alone. For this case, $\epst{t} = 0$, and the
inequality holds.
\end{proof}
\label{app:lemmas}
\begin{lemma}\label{lemma_(wtx-utx)2}
With the notation introduced so far, the following inequality holds for any $t$ and any $\vv\in\reals^{D}$, where we have $h_{s}=\vzti{s} \ai{s}^{-1} \vzi{s}$
\begin{align*}
    \left(\tvwti{t-1}\right.&\left.\vv-\vut\vv\right)^{2}=  2\vvt\ai{t-1}^{-1}\vv\cdot \\
    &\Big( d_{0}\paren{\vu,\vzero} + 2\sum_{s=1}^{t-1}h_{s} - \half\sum_{s=1}^{t-1}\paren{1-\tvwti{t-1}\vzi{s}}^2 + \half\sum_{s=1}^{t-1}\paren{1-\vut\vzi{s}}^2 \Big)
\end{align*}
\end{lemma}
\begin{proof}
Let $s$ be any round between round $1$ and round $t-1$. Using the Cauchy-Schwarz theorem for dual norms, we have,
\begin{align}
    \paren{\tvwti{s-1}\vv-\vut\vv}^{2}&=\abs{\inner{\tvwi{s-1}-\vu}{\vv}}^2  \leq 2\cdot\vvt\ai{s-1}^{-1}\vv\cdot d_{s-1}\paren{\tvwi{s-1}, \vu}. \label{eqn_(wtx-utx)2}
\end{align}
Next, similarly to \cite{DBLP:journals/ml/CrammerG13}, we have,
\begin{align*}
    d_{s}\paren{\tvwi{s-1},\vwi{s}} &= \half \paren{1 - \tvwti{s-1}\vzi{s}}^2 \cdot \vzti{s}\ai{s}^{-1}\vzi{s} \leq 2 \cdot \vzti{s}\ai{s}^{-1}\vzi{s}.
\end{align*}

Again, as shown in \cite{DBLP:journals/ml/CrammerG13},
observe that
\begin{align*}
    d_{s-1}\paren{\vu,\tvwi{s-1}} - d_{s}\paren{\vu,\vwi{s}} +
    d_{s}\paren{\tvwi{s-1},\vwi{s}}  =
 \half\paren{1-\tvwti{s-1}\vzi{s}}^2 - \half\paren{1-\vuti{s-1}\vzi{s}}^2
\end{align*}

Now, observe that since $\abs{\vut\vzii}\leq1$ holds for any $t=1, 2 \comdots T$, we can use the standard Bregman projection inequality stating that $d_{s}\paren{\vu,\tvwi{s}} \leq d_{s}\paren{\vu,\vwi{s}}$. Plugging this into the above, and summing over $s$, we get:
\begin{align*}
   \half\sum_{s=1}^{t-1} \paren{1-\tvwti{s-1}\vzi{s}}^2 -
    \half\sum_{s=1}^{t-1}\paren{1-\vut\vzi{s}}^2
\leq ~ d_{0}\paren{\vu,\vzero} - d_{t-1}\paren{\vu,\tvwi{s-1}} + 2\sum_{s=1}^{t-1} h_{s}.
\end{align*}
Isolating $d_{t-1}\paren{\tvwi{t-1},\vu}$, we get:
\begin{align*}
    d_{t-1} \paren{\vu,\tvwi{t-1}} \leq d_{0}\paren{\vu,\vzero} + 2
    \sum_{s=1}^{t-1} h_{s} -\half\sum_{s=1}^{t-1}\paren{1-\tvwti{s-1}\vzi{s}}^2  +\half\sum_{s=1}^{t-1}\paren{1-\vut\vzi{s}}^2\
\end{align*}
Plugging this back into \eqnref{eqn_(wtx-utx)2} concludes the proof.
\end{proof}

\begin{lemma}\label{lemma_E(wtx-utx)_q_t}
With the notation introduced so far, the following inequality holds for any (constant) $\vw\in\reals^{D}$, where $\vu\in\reals^{D}$ is the linear model approximating the reward function defined above:
\begin{equation*}
    \Exp{\half\paren{1-\vwt\vzii}^2-\half\paren{1-\vut\vzii}^2} \geq \half\paren{\hsigmaii-\sigmaii}^2-2q_t.
\end{equation*}
\end{lemma}
\begin{proof}
Let us define $\sigmaii=\half\paren{\deltatk{\indkt}-\deltatk{\indlt}}$, and $\hsigmaii=\half\paren{\hdeltatk{\indkt}-\hdeltatk{\indlt}}$. In addition, note that the probability of $\yii$ being $+1$, marked $p$, is $\half\paren{1+\half\paren{r\paren{t, \indmt} - r\paren{t, \indnt}}}$. Based on the setting described above, observe the following:
\begin{align*}
    \rm E \left\{ \paren{1-\vwt\vzii}^2 \right.&\left.- \paren{1-\vut\vzii}^2\right\} \\
    &= p \cdot \brackets{\paren{1-\hsigmaii}^2-\paren{1-\sigmaii}^2}
     + \paren{1-p} \cdot \brackets{\paren{1+\hsigmaii}^2-\paren{1+\sigmaii}^2} \\
    &= p \cdot  \brackets{\paren{2-\hsigmaii-\sigmaii}\cdot\paren{\sigmaii-\hsigmaii}} + \paren{1-p} \cdot \brackets{\paren{2+\hsigmaii+\sigmaii}\cdot\paren{\hsigmaii-\sigmaii}}\\
    &= 4p \cdot \paren{\sigmaii-\hsigmaii} + \brackets{\paren{2+\hsigmaii+\sigmaii}\cdot\paren{\hsigmaii-\sigmaii}}.
\end{align*}

Now, since
\(
    \left|p - \frac{1+\sigmaii}{2}\right|\leq\half q_t~,
\)
we have
\begin{align*}
    \rm E \left\{\right.& \left. \paren{1-\vwt\vzii}^2 - \paren{1-\vut\vzii}^2\right\}=4p \cdot \paren{\sigmaii-\hsigmaii} + \brackets{\paren{2+\hsigmaii+\sigmaii}\cdot\paren{\hsigmaii-\sigmaii}}\\
    &\geq 4\paren{\frac{1+\sigmaii}{2} - \half q_t} \cdot \paren{\sigmaii-\hsigmaii} + \brackets{\paren{2+\hsigmaii+\sigmaii}\cdot\paren{\hsigmaii-\sigmaii}}\\
    &=\paren{\hsigmaii-\sigmaii}^2 - 2q_t \cdot \paren{\sigmaii-\hsigmaii} \\
    &\geq \paren{\hsigmaii-\sigmaii}^2 - 2q_t \cdot \left|\sigmaii-\hsigmaii\right|\\
    &\geq \paren{\hsigmaii-\sigmaii}^2 - 4q_t ~.
\end{align*}
Adding the $\half$ coefficient concludes the proof.
\end{proof}

\begin{lemma}\label{lemma_delta_deltahat}
With the notation introduced so far, we have
\begin{align*}
    \left(\deltatk{\indmt, \indnt}\right.&\left.-\hdeltatk{\indmt,
        \indnt}\right)^2 \leq 2\vztii\attk{k}^{-1}\vzii\cdot \Bigg(d_0\paren{\vu,\vzero}+2\sum_{s=1}^{t} q_s + 2\sum_{s=1}^{t-1}h_s+36\cdot\ln\frac{t+4}{\delta}\Bigg),
\end{align*}
with probability at least $1-\delta$ uniformly over $t=1,2\dots$.
\end{lemma}
\begin{proof}
The proof is identical to the proof of a similar lemma in \cite{DBLP:journals/ml/CrammerG13}, substituting $\vxii$ for $\vzii$.
\end{proof}

\subsection{Proof of \thmref{thm_nonlinear}}
\label{app:proof_thm_nonlinear}

\begin{proof}
We have with probability at least $1-\delta$
\[
    R_T=\sum_{t=1}^{T}r_t \leq 2Q_t + 2 \sum_{t=1}^{T}\epst{t},
\]
where $\epst{t}^2 = 2\vztii\ai{t-1}^{-1}\vzii\cdot \eta_t$. Using previous techniques, such as the one shown in \cite{CesaBianchiCoGe05}, we have
\begin{align*}
\sum_{s=1}^{t-1}\vzti{s}\ai{s}^{-1}\vzi{s}
&\leq \ln \frac{\abs{\ai{t-1}}}{\abs{\ai{0}}}\\
&= \ln \abs{\ai{t-1}}\\
&\leq dK\cdot\ln\paren{1+\frac{\paren{t-1}}{dK}} \\
&\leq dK\cdot\ln\paren{1+\frac{T}{dK}}.
\end{align*}

Observe that at round $t$, the matrix $\ai{t-1}$ is augmented by the PSD rank-one matrix $\vzii\vztii$ regardless of the feedback. Hence,
\(
    \vztii\ai{t-1}^{-1}\vzii\leq1~.
\)
This means that $1\leq\frac{2}{1+\vztii\ai{t-1}^{-1}\vzii}$, and thus,
\[
\vztii\ai{t-1}^{-1}\vzii \leq \frac{2\cdot\vztii\ai{t-1}^{-1}\vzii} {1+\vztii\ai{t-1}^{-1}\vzii} = 2\vztii\aii^{-1}\vzii.
\]
We can therefore combine the results above:
\begin{align*}
\sum_{t=1}^{T}\epst{t}^2 &= \sum_{t=1}^{T}2\vztii\ai{t-1}^{-1}\vzii\times\paren{d_0\paren{\vu, \vzero}+2\sum_{s=1}^{t}q_t+2\sum_{s=1}^{t-1}h_s
+36\ln\frac{t+4}{\delta}}\\
&\leq \Bigg(d_0\paren{\vu, \vzero}+2\sum_{s=1}^{t}q_t+2dK\ln\paren{1+\frac{T}{dK}}+36\ln\frac{T+4}{\delta}\Bigg) \times 2\sum_{t=1}^{T}\vztii\ai{t-1}^{-1}\vzii \\
&\leq \Bigg(d_0\paren{\vu, \vzero}+2\sum_{s=1}^{t}q_t+2dK\ln\paren{1+\frac{T}{dK}} +36\ln\frac{T+4}{\delta}\Bigg) \times 4\sum_{t=1}^{T}\vztii\ai{t}^{-1}\vzii \\
&\leq \Bigg(d_0\paren{\vu, \vzero}+2\sum_{s=1}^{t}q_t+2dK\ln\paren{1+\frac{T}{dK}} +36\ln\frac{T+4}{\delta}\Bigg) \times 4dK\ln\paren{1+\frac{T}{dK}}.
\end{align*}
Observe that since $\sum_{t=1}^{T}\epst{t}^2\leq M$, it holds that $\sum_{t=1}^{T}\epst{t}\leq \sqrt{T\cdot M}$. Combining this with all of the above, we get
\begin{align*}
R_T &\leq~ 2Q_T + 2\sum_{t=1}^{T}\epst{t}\\
&\leq~2Q_T + \sqrt{d_0\paren{\vu, \vzero}+2Q_T
  +2dK\ln\paren{1+\frac{T}{dK}}+36\ln\frac{T+4}{\delta}} \times \sqrt{4TdK\ln\paren{1+\frac{T}{dK}}}\\
&= 2Q_T + \sqrt{2T}\paren{\sqrt{\paren{(2Q_T + A)B}} + B},
\end{align*}
thus proving the regret bound.
\end{proof}
\subsection{Equivalence to the Noisy Reward Setting}
\label{sec:noisy_reward}
As mentioned before, we assume that the rewards $r\paren{t,m}$ are as before,
but feedback is given based on a noisy version of the reward, $\hat{r}\paren{t,m}$,
such that:
\[
    \hat{r}\paren{t,m} = \begin{cases}
        + 1     &   w.p. \frac{1+r\paren{t,m}}{2} \\
        - 1     &   w.p. \frac{1-r\paren{t,m}}{2}
    \end{cases}
\]

meaning, the reward of an instance is shifted to either $+1$ or $-1$ with bias
according to the unshifted reward. As for the feedback, we assume now that it is
deterministic rather than stochastic, depending on the values of the noisy rewards
of the two predictions, with one exception - if the rewards are identical (both $+1$
or $-1$), the feedback is $\pm 1$ with equal probabilities. Let's now calculate
the probability of receiving a $+1$ reward for predictions $\indmt, \indnt$:

\begin{align*}
    \Pr\brackets{\yii = +1} &= \Pr\brackets{\hat{\rho}\paren{t,\indmt} = +1 \land \hat{\rho}\paren{t,\indmt} = -1} \\
                            &+ \half\cdot\Pr\brackets{\hat{\rho}\paren{t,\indmt} = +1 \land \hat{\rho}\paren{t,\indmt} = +1} \\
                            &+ \half\cdot\Pr\brackets{\hat{\rho}\paren{t,\indmt} = -1 \land \hat{\rho}\paren{t,\indmt} = -1} \\
                            &= \paren{\frac{1+\rho\paren{t,\indmt}}{2}}\paren{\frac{1-\rho\paren{t,\indnt}}{2}} \\
                            &+ \half\cdot\paren{\frac{1+\rho\paren{t,\indmt}}{2}}\paren{\frac{1+\rho\paren{t,\indnt}}{2}} \\
                            &+ \half\cdot\paren{\frac{1-\rho\paren{t,\indmt}}{2}}\paren{\frac{1-\rho\paren{t,\indnt}}{2}} \\
                            &= \frac{1+\half\paren{\rho\paren{t,\indmt}-\rho\paren{t,\indnt}}}{2}
\end{align*}

This is of course equivalent to the case where the rewards are clean and the
feedback is stochastic we've shown before. Our regret now can either take into
account the clean rewards as before, or the \emph{expected} rewards, as follows:

\[
    r_t = \max_m \Exp{\hat{\rho}\paren{t,m}} - \max\braces{\Exp{\hat{\rho}\paren{t,\indmt}},\Exp{\hat{\rho}\paren{t,\indnt}}}
\]

Notice the following:
\begin{align*}
    \Exp{\hat{\rho}\paren{t,m}} &= (+1)\cdot\Pr\brackets{\hat{\rho}\paren{t,m} = +1} + (-1)\cdot\Pr\brackets{\hat{\rho}\paren{t,m} = -1} \\
                                &= \frac{1+\rho\paren{t,m}}{2}-\frac{1-\rho\paren{t,m}}{2} \\
                                &= \rho\paren{t,m}
\end{align*}

Meaning the regret is also identical to before. This implies
that the regret bound we've shown for the stochastic feedback
setting applies here as well.

\subsection{Dependency Parsing}
\label{sec:sup_dependency_parsing}

In this setting, given a sentence $\vxii = \paren{\vxi{t,1} \comdots
  \vxi{t,n_t}}$ in some natural language, we aim to find the best
dependency parsing tree of that sentence. We assume a linear
regret model, where the reward for a parsing $\vmii$ for sentence
$\vxii$ is given by $\vut\vPhii{\vxii, \vmii}$ for some $\vu \in \reals^D$ in the unit ball
($\| \vu \| \leq 1$). This is done both for
simplicity and since parse trees are often evaluated by the number of
mistakes. As before, our algorithm maintains a weight vector $\vwii\in\reals^D$,
which is used to estimate the optimal dependency parsing. In addition,
it maintain a PD matrix $\aii$, which is used to estimate our
confidence in the relation between two possible parses. Finally, let $\tvwii$ denote
an orthogonal projection of $\vwii$ that maintains $\vert \tvwtii \vPhii{\vxii, \vm} \vert \leq 1$
for any $\vm \in \mathcal{K}_t$.

We find the optimal parsing by maximizing the following term:
\[
    \vmii = \arg\max_{\vm\in\mathcal{K}_t} \tvwti{t-1} \vPhii{\vxii, \vm}.
\]
In order to find the maximal $\vm$, we construct from the sentence a full (directed) graph, where each word is a node, and the weight of an edge from word $\vxin{t}{i}$ to word $\vxin{t}{j}$ is given by $\tvwti{t-1}\vphii{\vxii, i, j}$. We can then use the Chu-Liu-Edmonds algorithm \cite{edmonds1968optimum} to find the (non-projective) parsing $\vm$ which maximizes the term above, in $O\paren{S^2}$ time.

Let $\pi_{\vm}\paren{j}$ denote the source of word $\vxin{t}{j}$
according to $\vm$. After producing the believed-best parse tree, we
select a word in the sentence (indexed by $\indit$) and a possible
alternative source for it (indexed by $\indjt$), such that our
confidence in the source of $\vxin{t}{\indit}$ being
$\pi_{\vmii}\paren{\indit}$ or $\vxin{t}{\indjt}$ is minimal, meaning
that the confusion term is maximal. Algorithm~\ref{conquer_dependency}
describes our solution for this setting.

For convenience, we define the following notations:
\begin{align*}
    &\epstk{i, j, k}^2 &=& \eta_t \times \tran{ \paren{ \vphii{t, i,
          j} - \vphii{t, k, j} } } \ai{t-1}^{-1}\\
&&&\times \paren{ \vphii{t, i, j} - \vphii{t, k, j} } \\
    &\deltatk{i, j} &=& \vut\vphii{\vxii, i, j} \\
    &\deltatk{i, j, k} &=& \half \paren{\deltatk{i, j} - \deltatk{k, j}} \\
    &\hdeltatk{i, j} &=& \tvwti{t-1}\vphii{\vxii, i, j} \\
    &\hdeltatk{i, j, k} &=& \half \paren{\hdeltatk{i, j} - \hdeltatk{k, j}}
\end{align*}

\begin{algorithm}[!ht]
\caption{\cnqrname\ Dependency Parsing - Algorithm Outline}
\label{conquer_dependency}
\begin{algorithmic}[1]
\REQUIRE $\delta \in \paren{0, 1}$
\STATE Initialize $\vwi{0} = \vzero \in \reals^{D}, \ai{0} = \mi_{D\times D}$
\FOR {$t = 1$ to $T$}
    \STATE Receive sentence $\vxii = \paren{\vxi{t,1} \comdots \vxi{t,s_t}}$
    \STATE Project $\tvwi{t-1} = \arg\min ~ d_t\paren{\vw, \vwi{t-1}}$ \\
    Subject To $\vert\vwt\vphii{\vxii, i,j} \vert \leq 1 ~ \forall i,j \in \braces{1\comdots s_t},~ i\neq j$
    \STATE Set $\vmii = \arg\max_{\vm \in \mathcal{K}_t} \hdeltatk{\vm}$
    \FORALL {$\paren{k, j} \in \vmii$}
        \FORALL {$i \in {1 \comdots s_t},~i\neq k,j$}
            \STATE Set $\beta\paren{i, j} = \epstk{k, j, i}$
        \ENDFOR
    \ENDFOR
    \STATE Set $\paren{\indit, \indjt} = \underset{i,j}{\arg\max} \beta\paren{i, j}$
    \STATE Set $\vnii$ such that $\pi_{\vnii}\paren{j} =
                    \begin{cases}
                        \pi_{\vmii}\paren{j}  &\mbox{if } j\neq \indjt \\
                        \indit          &\mbox{if } j= \indjt
                    \end{cases}
                    $
    \STATE Output $\vmii, \vnii$
    \STATE Receive feedback $\yii \in \braces{ \pm 1 }$
    \STATE Set $\vzii = \half \yii \cdot \paren{\vPhii{\vxii, \vmii} - \vPhii{\vxii, \vnii}}$
    \STATE Update $\vwii = \aii^{-1}\paren{\ai{t-1} \tvwi{t-1} + \vzii}$
    \STATE Update $\aii = \ai{t-1} + \vzii \vztii$
\ENDFOR
\ENSURE $\tvwi{T}$
\end{algorithmic}
\end{algorithm}

\subsubsection*{Regret Bound for the Dependency Parsing Setting}
This setting is slightly different than the settings we have seen
previously, but we can build on most of the results or proof techniques.

\begin{lemma}\label{lemma_r_t_parse}
In the dependency parsing setting, if at time $t$ the algorithm is such that $\abs{\deltatk{i,j,k}-\hdeltatk{i,j,k}} \leq \epstk{i,j,k}$ for all $i,j,k \in \braces{1 \comdots s_t}$, then
\[
    r_{t}= \underset{\vm\in\mathcal{K}_t}{\max}~\deltatk{\vm} - \max\braces{\deltatk{\vmii}, \deltatk{\vnii}} \leq  2S\epst{t}
\]
where
\[
    \epst{t}^2 = 2\vztii\ai{t-1}^{-1}\vzii\times\eta_t
\]
for all $t = 1 , 2 \dots$.
\end{lemma}
\begin{proof}
Similarly to \lemref{lemma_r_t_q_t}, we have:
\begin{align*}
    r_t &=~ \deltats -\max\braces{\deltatk{\vmii},\deltatk{\vnii}} \\
    & \leq~   \deltats -\deltatk{\vmii} \\
    &=~ 2\sum_{i=1}^{s_t} \deltatk{\pi_{\vmis}\paren{i}, i, \pi_{\vmii}\paren{i}}\\
    &\leq~  2\sum_{i=1}^{s_t} \hdeltatk{\pi_{\vmis}\paren{i}, i, \pi_{\vmii}\paren{i}} + \epstk{\pi_{\vmis}\paren{i}, i, \pi_{\vmii}\paren{i}} \\
    &=~ \hdeltats - \hdeltatk{\vmii} + 2\sum_{i=1}^{s_t} \epstk{\pi_{\vmis}\paren{i}, i, \pi_{\vmii}\paren{i}}\\
    &\leq~  2S \cdot \epstk{\indit, \indjt, \pi_{\vmii}\paren{\indjt}} \\
    &=~ 2S \times \epst{t}~.
\end{align*}
\end{proof}

\begin{theorem}\label{thm_parse}
In the dependency parsing setting described so far, the cumulative regret $R_t$ of the algorithm satisfies
\[
    R_T = O \paren{S\times \sqrt{2T}\paren{\sqrt{AB}+B}},
\]
where
\begin{eqnarray*}
    A &= &d_0\paren{\vu,\vzero}+36\cdot\ln\frac{T+4}{\delta}\\
    B &= &2dK\ln\paren{1+\frac{T}{dK}},
\end{eqnarray*}
with probability at least $1-\delta$ uniformly over the time horizon $T$.
\end{theorem}
\begin{proof}
From \lemref{lemma_r_t_parse}, we have the following result:
\[
    R_T=\sum_{t=1}^{T}r_t \leq 2S \sum_{t=1}^{T}\epst{t}.
\]
We can now use the exact same path as in \thmref{thm_nonlinear} to show the regret bound.
\end{proof}

\subsection{Computing the Projection}
\label{sec:projection}
Denote by $\vv^i$ the vector
computed after the i$th$ iteration. Then, the new vector at this point
is given by,
\(
\vv^{i+1} = \arg\min_\vw~d_t\paren{\vw, \vv^i} ~ s.t. ~
\vert\vwt\vPhii{\vxii, m(i)} \vert \leq 1~,
\)
for some $m(i) \in\mathcal{K}$, which can be easily solved using
Lagrange's method. The solution is given by,
\begin{align}
\vv^{i+1}   &= \vv^{i} - \aii^{-1} \vPhii{\vxii, m(i)} \cdot \sign\paren{\vv^i\cdot\vPhii{\vxii,  m(i)}}\label{projection}
            \cdot\frac{\max\{ \vert \vv^i\cdot\vPhii{\vxii,   m(i)}\vert-1,0\}}{\vPhii{\vxii, m(i)}^\top \aii^{-1}\vPhii{\vxii, m(i)}}  ~.
\end{align}

To conclude, the projection algorithm sets $\vv^1=\vwi{t-1}$. It then
iterates. On the i$th$ iteration, it picks a (random) index $m=m(i)$
and sets $\vv^{i+1}$ according to the three cases of $\vv^i$ sketched
above. This process is repeated until some convergence criteria are
met, and then the last value vector $\vv^i$ is returned. The time
complexity is $O(D^2 + n D^2)$, where $D^2$ is the time needed to
compute the inverse $\aii^{-1}$ if performed incrementally, $D^2$
the time needed to compute each iteration, and $n$ the number of iterations performed.

\subsection{Additional Related Work}
\tabref{table_regret_summary} and \tabref{table_summary} provide some context
for our setting compared with other similar settings, enabling a side-by-side
comparison and a better understanding of our solution compared with existing works.

\begin{table*}
    \begin{tabular}{|>{\tiny}l|>{\tiny}l|>{\tiny}l|>{\tiny}l|} \hline
    ~                       & \textbf{\algname -GNC} (ours)                               & \textbf{Confidit} \cite{DBLP:journals/ml/CrammerG13}  & \textbf{Partial Feedback} \cite{gentile2012multilabel}\\\hline \hline
    Parameters              & None                                                      & $\alpha \in (-1, 1]$                                  & Constant $a\in\brackets{0, 1}$                                                \\
    ~                       & ~                                                         & ~                                                     & Cost values $c\paren{i,s}$                                                    \\
    ~                       & ~                                                         & ~                                                     & Interval $D\in\brackets{-R, R}$                                               \\
    ~                       & ~                                                         & ~                                                     & Function $g:D\rightarrow\mathcal{R}$                                          \\ \hline
    Model                   & \multicolumn{2}{>{\footnotesize}c|}{$\vwii\in\reals^{D}$}                                                                 & $\vwi{i,t}\in\reals^{d}~,~i = 1 \comdots K$                                                      \\
    ~                       & \multicolumn{2}{>{\footnotesize}c|}{$\aii\in\reals^{D\times D}$}                                                          & $\ai{i,t}\in\reals^{d\times d}~,~i = 1 \comdots K$                                               \\ \hline
    Predictions             & $\indmt = \arg\max_m \hdeltatk{m}$                        & $\indmt = \arg\max_m \hdeltatk{m} + \epstk{m}$        & $\hat{Y}_t = \arg\max_Y $                                                     \\
    ~                       & $\indnt = \arg\max_n \hdeltatk{n} + \epstk{\indmt, n}$    & ~                                                     & $~~~\paren{1-a} \sum_{i\in Y}\left(c\paren{j_i, \abs{Y}} \right.$             \\
    ~                       & ~                                                         & ~                                                     & $~~~\left.- \paren{\frac{a}{1-a}+c\paren{j_i, \abs{Y}}} \hat{p}_{i,t}\right)$ \\ \hline
    Number of               & $2$                                                       & $1$                                                   & $1 \leq \abs{\hat{Y}_t} \leq K$                                               \\
    Predictions             & ~                                                         & ~                                                     & ~                                                                             \\ \hline
    Feedback                & Stochastic (Bernoulli)                                    & \multicolumn{2}{>{\footnotesize}c|} {Deterministic}                                                                                           \\
    Type                    & ~                                                         & \multicolumn{2}{>{\footnotesize}c|} {~}                                                                                                       \\ \hline
    Feedback                & $\Pr\brackets{\yii=\pm 1}$ & $\yii = \begin{cases} 1 & \indst \neq \indmt\\ 0 & \indst = \indmt \end{cases}$ & $Y_t\cap\hat{Y}_t$                                 \\
    ~                       & $~~~= \half\paren{1\pm\half\paren{r\paren{t, \indmt} - r\paren{t, \indnt}}}$  & ~      & ~                                                                             \\ \hline
    Update                  & ~                                                         & $s_t = \begin{cases} +1 & \yii = 0 \\ +1 & \yii = 1,\textrm{w.p.} \frac{1-\alpha}{2} \\ -1 & \yii = 1,\textrm{w.p.} \frac{1+\alpha}{2} \end{cases}$              & $s_{i,t} = \begin{cases} +1 & i\in Y_t \cap \hat{Y}_t \\ -1 & i\in \hat{Y}_t / Y_t \\ 0 & \textrm{otherwise}\end{cases}$ \\
    ~                       & $\vzii = \half\paren{\vphii{\vxii, \indmt}-\vphii{\vxii, \indnt}}$  & $\vzii = \vphii{\vxii, \indmt}$             & $\vzii = \vxii$                                                   \\
    ~                       & $\aii = \ai{t-1} + \vzii \vztii$                          & $\aii = \ai{t-1} + \vzii \vztii$                      & $\ai{i,t} = \ai{i,t-1} + \vzii\vztii$                                 \\
    ~                       & $\vwii = \tvwi{t-1}$                                      & $\vwii = \tvwi{t-1}$                                  & $\vwi{i,t} = \tvwi{i,t}$                                                      \\
    ~                       & $~~~+\paren{\yii-\hdeltatk{\indmt, \indnt}}\aii^{-1}\vzii$               & $~~~+ \paren{s_t-\hdeltatk{\indmt}}\aii^{-1}\vzii$           & $~~~+ \frac{g\paren{s_{i,t}\hdeltatk{i,t}}s_{i,t}}{c_L^{''}} \ai{i,t}^{-1}\vzii$ \\ \hline
    Output                  & \multicolumn{3}{>{\footnotesize}c|}{$\vwi{T}$} \\ \hline
    \end{tabular}
    \caption{A summary of contextual bandits methods with partial feedback}
    \label{table_summary}
\end{table*}

\begin{table*}
    \begin{tabular}{| >{\tiny}l | >{\tiny}l | >{\tiny}l | >{\tiny}l | >{\tiny}l | >{\tiny}l |} \hline
    \textbf{Algorithm}              & \textbf{Regret}                                                       & \textbf{Feedback} & \textbf{Assumptions}  & \textbf{Algorithm}   & \textbf{Context}          \\
    \textbf{Name}                   & \textbf{Formulation}                                                  & \textbf{Type}     & ~                     & \textbf{Type}        & ~                         \\ \hline\hline
    \textit{Interleaved}            & \textit{Strong regret:}                                               & Relative          & - Stochastic Triangle & Survival             & No                        \\
    \textit{Filter} \cite{yue2012k} & $\Pr\braces{b_{\indst} > b_{\indmt}} + \Pr\braces{b_{\indst} > b_{\indnt}} $ & Stochastic & ~~Inequality          & ~                    & ~                         \\
    ~                               & \textit{Weak regret:}                                                 & ~                 & - Strong Stochastic   & ~                    & ~                         \\
    ~                               & $\min\left[\Pr\braces{b_{\indst} > b_{\indmt}}, \Pr\braces{b_{\indst} > b_{\indnt}}\right]$ & ~ & ~~Transitivity  & ~                    & ~                         \\ \hline
    \textit{Beat the}               & $\half\Pr\braces{b_{\indst} > b_{\indmt}}+\half\Pr\braces{b_{\indst} > b_{\indnt}}$ & Relative  & - Stochastic Triangle & Survival       & No                        \\
    \textit{Mean} \cite{yue2011beat}& ~                                                                     & Stochastic        & ~~Inequality          & ~                    & ~                         \\
    ~                               & ~                                                                     & ~                 & - Relaxed Stochastic  & ~                    & ~                         \\
    ~                               & ~                                                                     & ~                 & ~~Transitivity        & ~                    & ~                         \\ \hline
    \textit{Confidit} \cite{DBLP:journals/ml/CrammerG13} & $\max_m\brackets{\Pr\paren{\indst = m} }$        & Absolute          & Generative model      & UCB                  & Yes                       \\
    ~                               & $~~~~~~~~~- \Pr\paren{\indst = \indmt}$                               & Deterministic     & for labels            & Second Order         & ~                         \\ \hline
    \textit{Partial}                & $\paren{1-a}\sum_{i\in\hat{Y}_t}\left(c\paren{j_i, \abs{\hat{Y}_t}}\right.$ & Absolute    & Specific form for     & Linear               & Yes                       \\
    \textit{Feedback} \cite{gentile2012multilabel} & $~~~~~~~~~\left.-\paren{\frac{a}{1-a}+c\paren{j_i, \abs{\hat{Y}_t}}} \mathbbm{1}_{i\in Y_t}\right)$& Deterministic        & marginals  & Second Order  & ~ \\ \hline
    \textit{\algname -GNC}            & $\min \left[r\paren{\vxii, \indst}- r\paren{\vxii, \indmt}, \right. $  & Relative         & Reward can be         & Mixed (UCB           & Yes                       \\
    (ours)                          & $~~~~~~~~~\left.r\paren{\vxii, \indst}-r\paren{\vxii, \indnt}\right] $ & Stochastic       & approximated using    & \& Linear)           & ~                         \\
    ~                               & ~                                                                      & linear model     & Second Order          & ~                    & ~                          \\ \hline
    \end{tabular}
    \caption{A comparison of bandits methods with partial feedback}
    \label{table_regret_summary}
\end{table*}

\subsection{Additional Results} \label{app:results}
We include additional results from our study, which could not be fitted into the paper itself.
\figref{fig_amazon_test2} shows the test errors of the four tested algorithms, over an additional
9 domains. These results correspond to the average and aggregated results we included in the paper.
Finally, \figref{fig_amazon_train} demonstrates the performance of the tested algorithms in
terms of online regret over 3 domains. It is evident that good performance in terms of test error
corresponds with good performance in terms of online loss, which is an indication that online loss
is a valid estimate of algorithm performance in this setting.

Finally, \tabref{tab:optional_query} shows in what percent of rounds \algname -GNC avoided making
a query on average, per domain and number of reviews per round.

\begin{table*}
\begin{center}
    \begin{tabular}{|l|c|c|c|c|} \hline
\textbf{Number of Reviews} & \textbf{5} & \textbf{10} & \textbf{15} & \textbf{20} \\ \hline \hline
\textit{             Android} & 0.53\% & 0.79\% & 0.53\% & 0.26\% \\ \hline
\textit{                Arts} & 0.26\% & 0.00\% & 0.00\% & 0.00\% \\ \hline
\textit{          Automotive} & 0.64\% & 0.21\% & 0.00\% & 0.00\% \\ \hline
\textit{       Baby Products} & 1.73\% & 0.11\% & 0.05\% & 0.05\% \\ \hline
\textit{              Beauty} & 2.00\% & 0.16\% & 0.11\% & 0.05\% \\ \hline
\textit{               Books} & 0.07\% & 0.01\% & 0.01\% & 0.00\% \\ \hline
\textit{              Camera} & 0.59\% & 0.29\% & 0.03\% & 0.08\% \\ \hline
\textit{         Cell Phones} & 1.17\% & 0.11\% & 0.05\% & 0.03\% \\ \hline
\textit{            Clothing} & 4.56\% & 0.83\% & 0.32\% & 0.16\% \\ \hline
\textit{           Computers} & 0.27\% & 0.00\% & 0.05\% & 0.00\% \\ \hline
\textit{         Electronics} & 0.76\% & 0.19\% & 0.13\% & 0.07\% \\ \hline
\textit{                Food} & 1.41\% & 0.37\% & 0.11\% & 0.13\% \\ \hline
\textit{      Garden \& Pets} & 1.23\% & 0.16\% & 0.05\% & 0.00\% \\ \hline
\textit{              Health} & 2.01\% & 0.41\% & 0.12\% & 0.08\% \\ \hline
\textit{                Home} & 0.59\% & 0.21\% & 0.03\% & 0.00\% \\ \hline
\textit{          Industrial} & 0.00\% & 0.00\% & 0.00\% & 0.00\% \\ \hline
\textit{             Jewelry} & 0.53\% & 0.26\% & 0.26\% & 0.00\% \\ \hline
\textit{              Kindle} & 0.24\% & 0.05\% & 0.00\% & 0.01\% \\ \hline
\textit{             Kitchen} & 1.96\% & 0.43\% & 0.08\% & 0.01\% \\ \hline
\textit{           Magazines} & 0.00\% & 0.00\% & 0.00\% & 0.00\% \\ \hline
\textit{        Movies \& TV} & 0.48\% & 0.21\% & 0.12\% & 0.01\% \\ \hline
\textit{                 MP3} & 0.40\% & 0.12\% & 0.04\% & 0.11\% \\ \hline
\textit{               Music} & 1.21\% & 0.59\% & 0.41\% & 0.12\% \\ \hline
\textit{ Musical Instruments} & 0.00\% & 0.00\% & 0.00\% & 0.00\% \\ \hline
\textit{              Office} & 1.01\% & 0.16\% & 0.16\% & 0.00\% \\ \hline
\textit{               Patio} & 0.80\% & 0.19\% & 0.08\% & 0.00\% \\ \hline
\textit{               Shoes} & 8.27\% & 1.44\% & 0.75\% & 0.51\% \\ \hline
\textit{            Software} & 0.48\% & 0.05\% & 0.03\% & 0.00\% \\ \hline
\textit{              Sports} & 0.51\% & 0.19\% & 0.11\% & 0.05\% \\ \hline
\textit{                Toys} & 1.65\% & 0.43\% & 0.69\% & 0.43\% \\ \hline
\textit{         Video Games} & 0.48\% & 0.00\% & 0.16\% & 0.05\% \\ \hline
\textit{              Videos} & 0.12\% & 0.01\% & 0.00\% & 0.01\% \\ \hline
\textit{             Watches} & 0.00\% & 0.00\% & 0.00\% & 0.00\% \\ \hline
\end{tabular}
\end{center}
    \caption{Percent of rounds where query was not made in \algname -GNC for each domain}
    \label{tab:optional_query}
\end{table*}

\begin{figure*}[ht]
 \centering
  \includegraphics[width=0.3\textwidth]{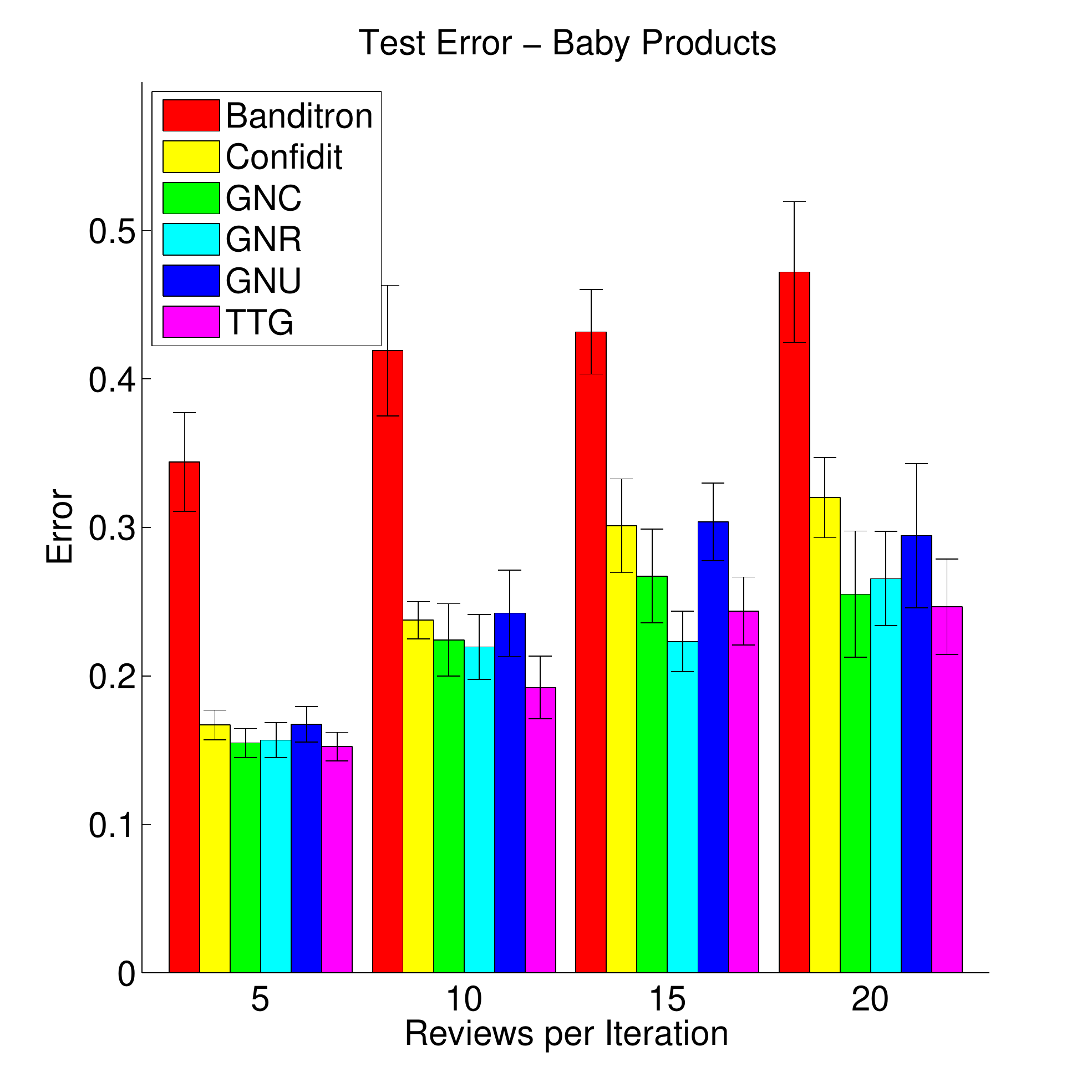}
  \includegraphics[width=0.3\textwidth]{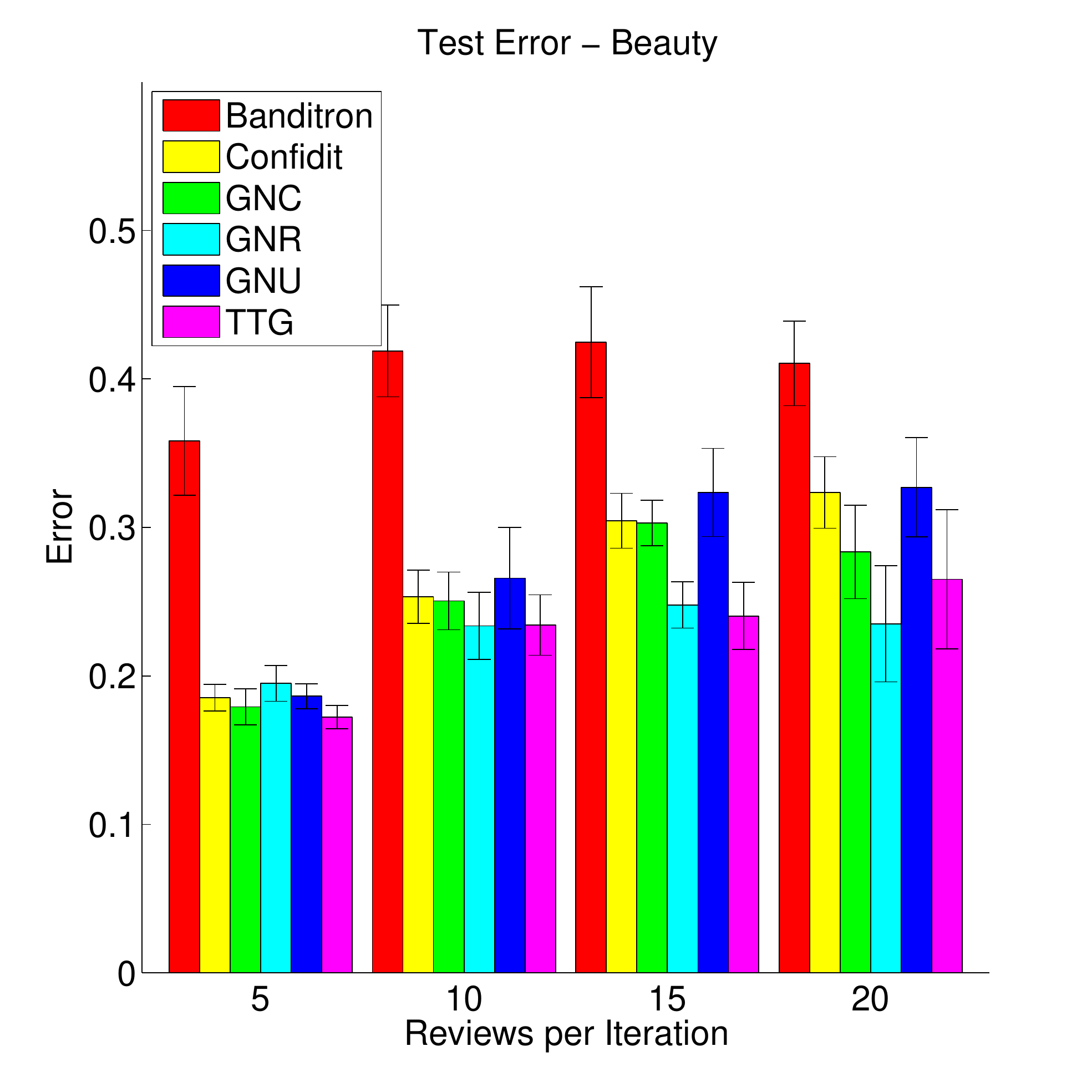}
  \includegraphics[width=0.3\textwidth]{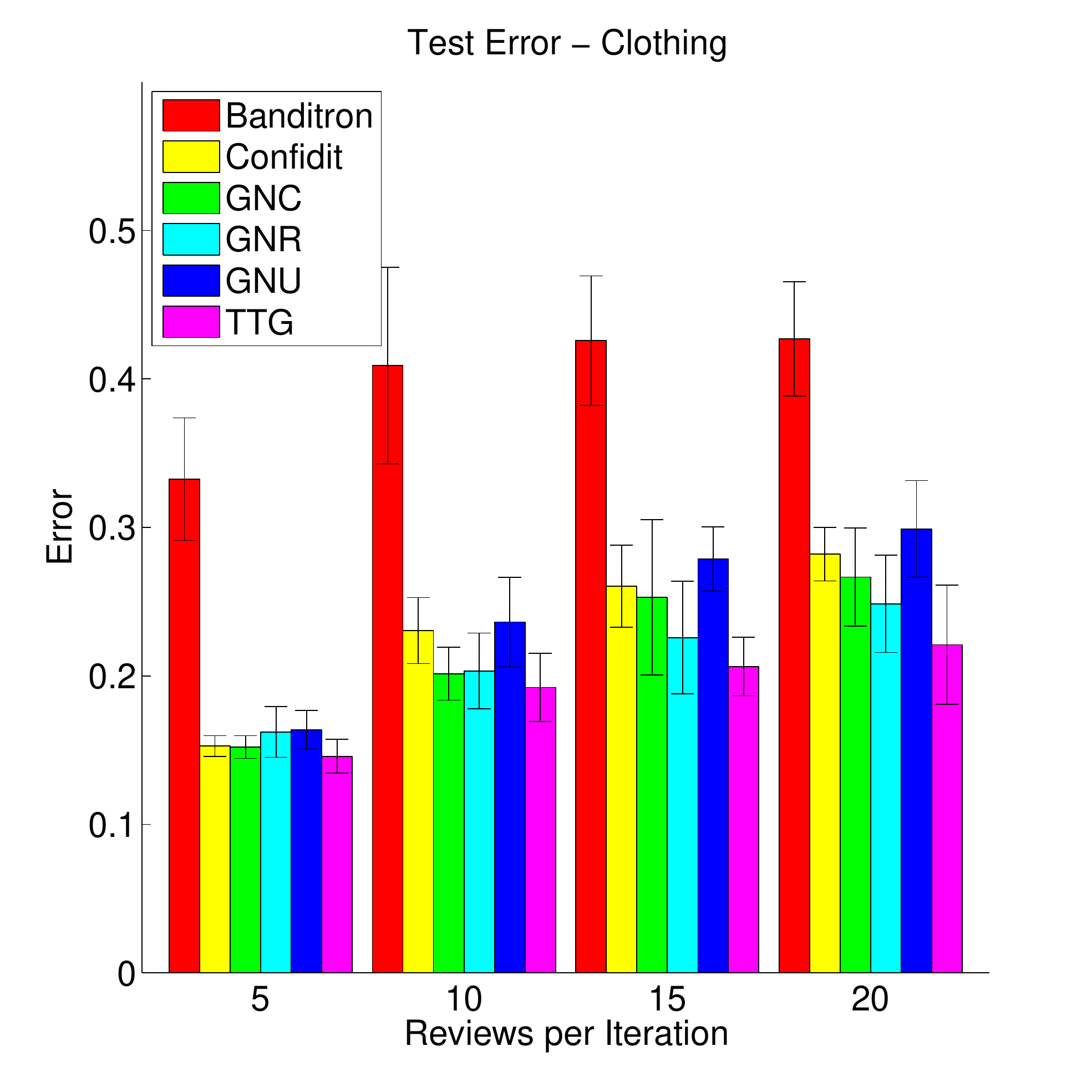}
  \includegraphics[width=0.3\textwidth]{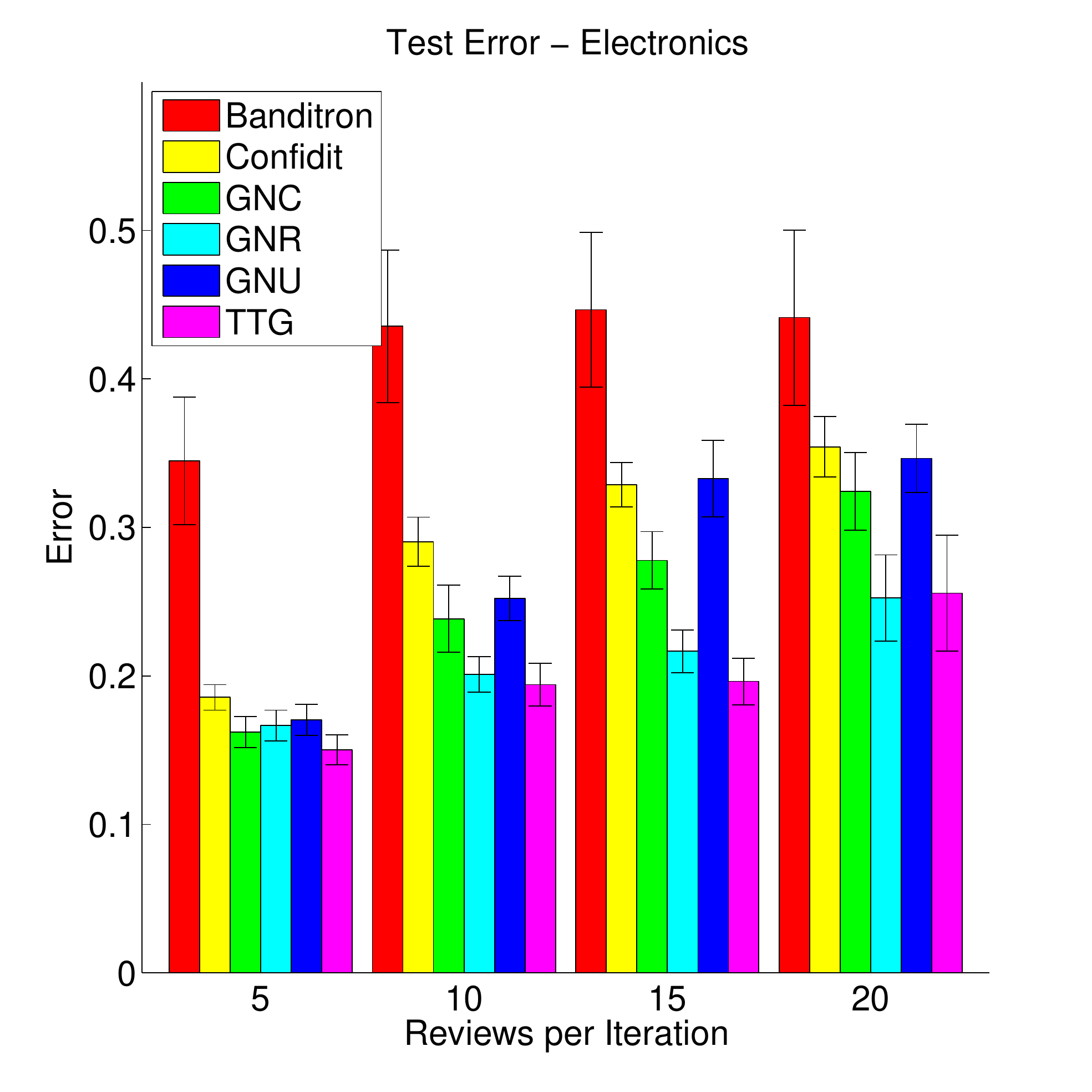}
  \includegraphics[width=0.3\textwidth]{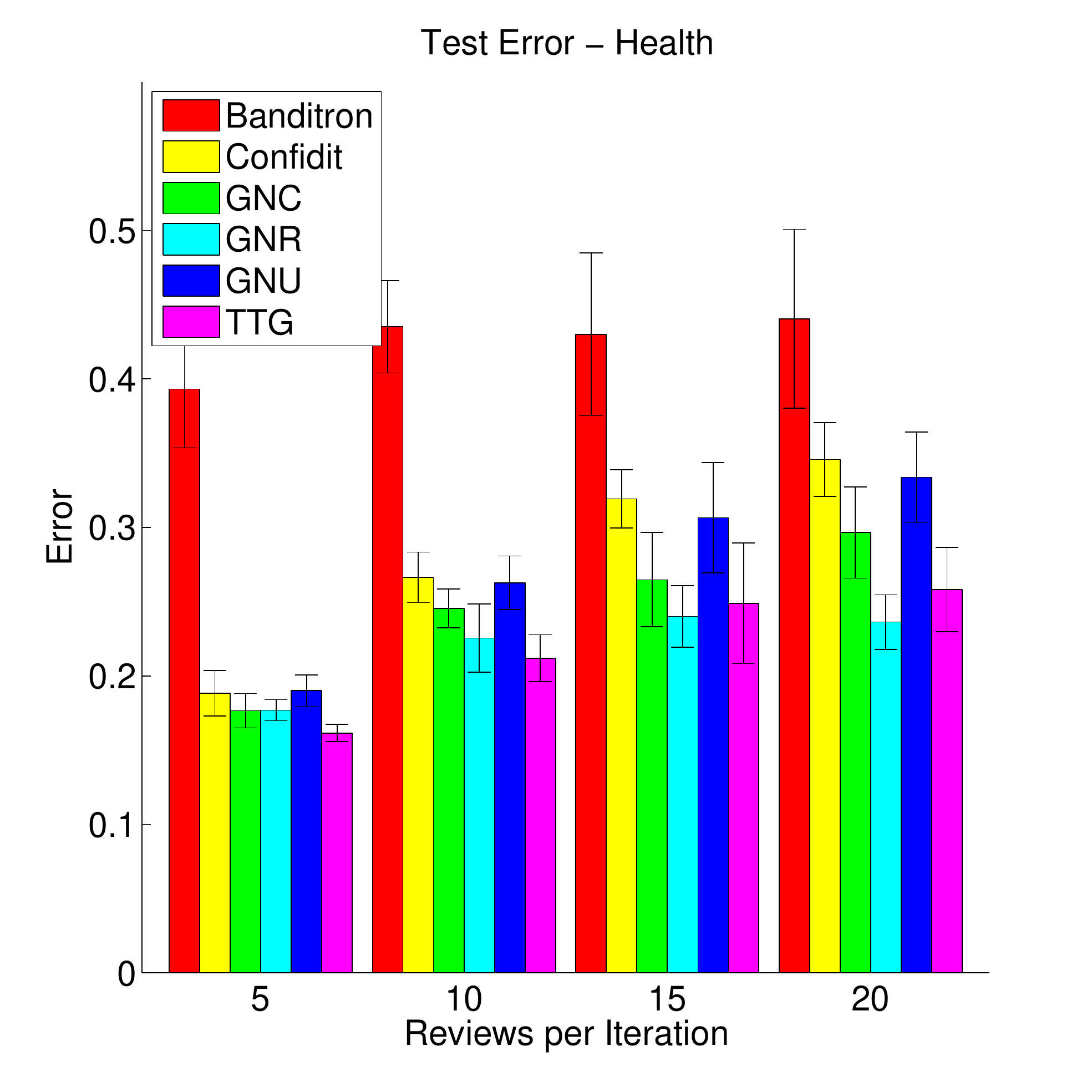}
  \includegraphics[width=0.3\textwidth]{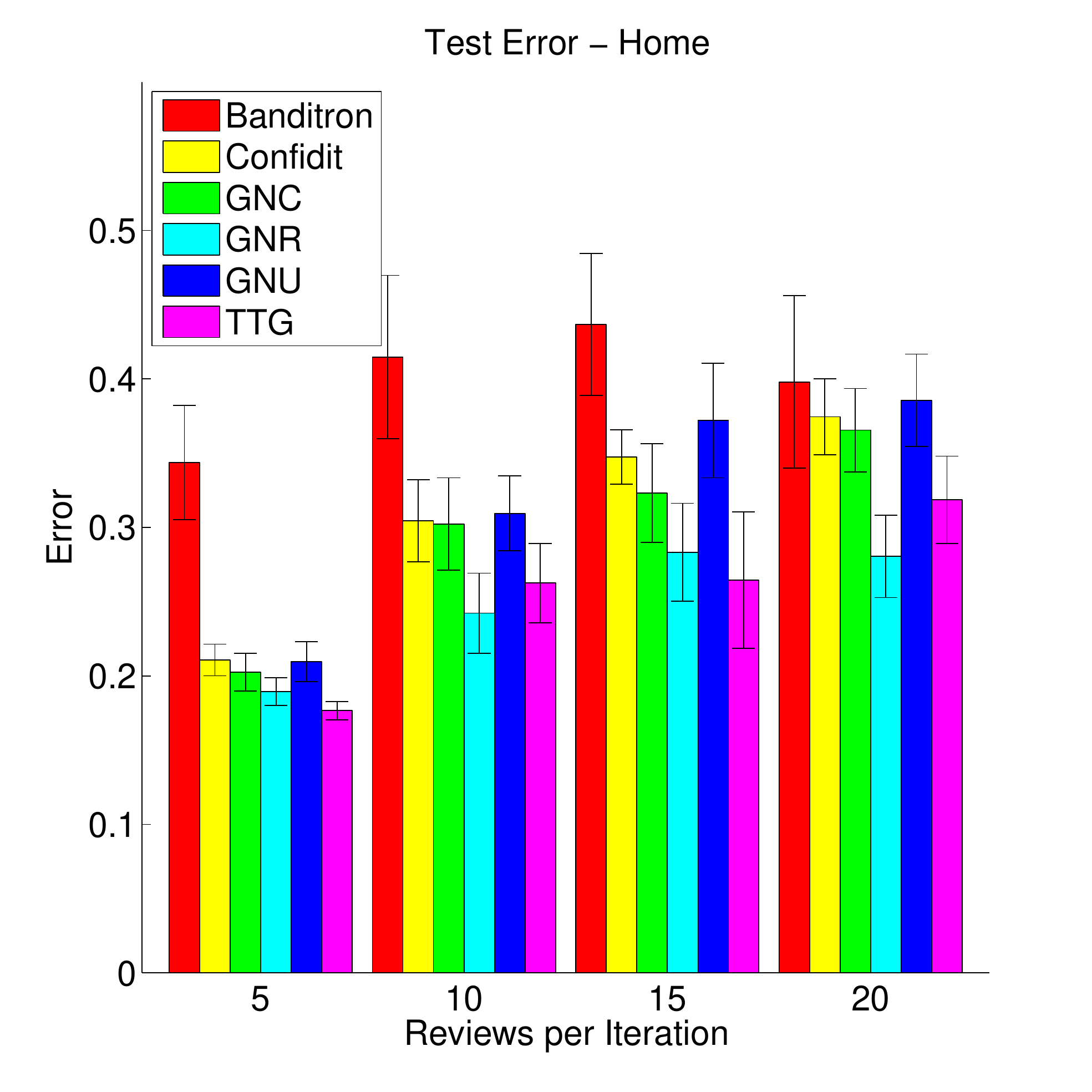}
  \includegraphics[width=0.3\textwidth]{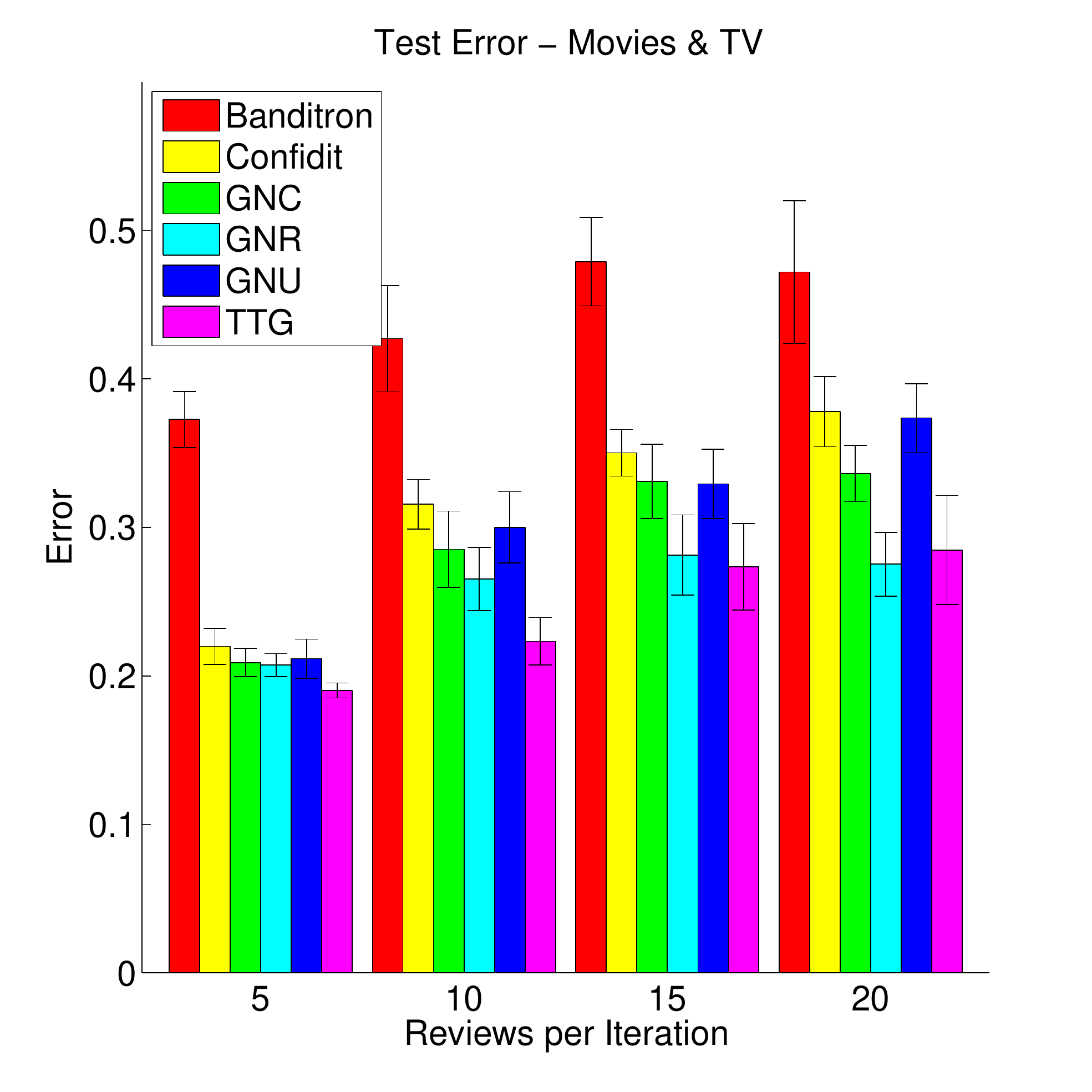}
  \includegraphics[width=0.3\textwidth]{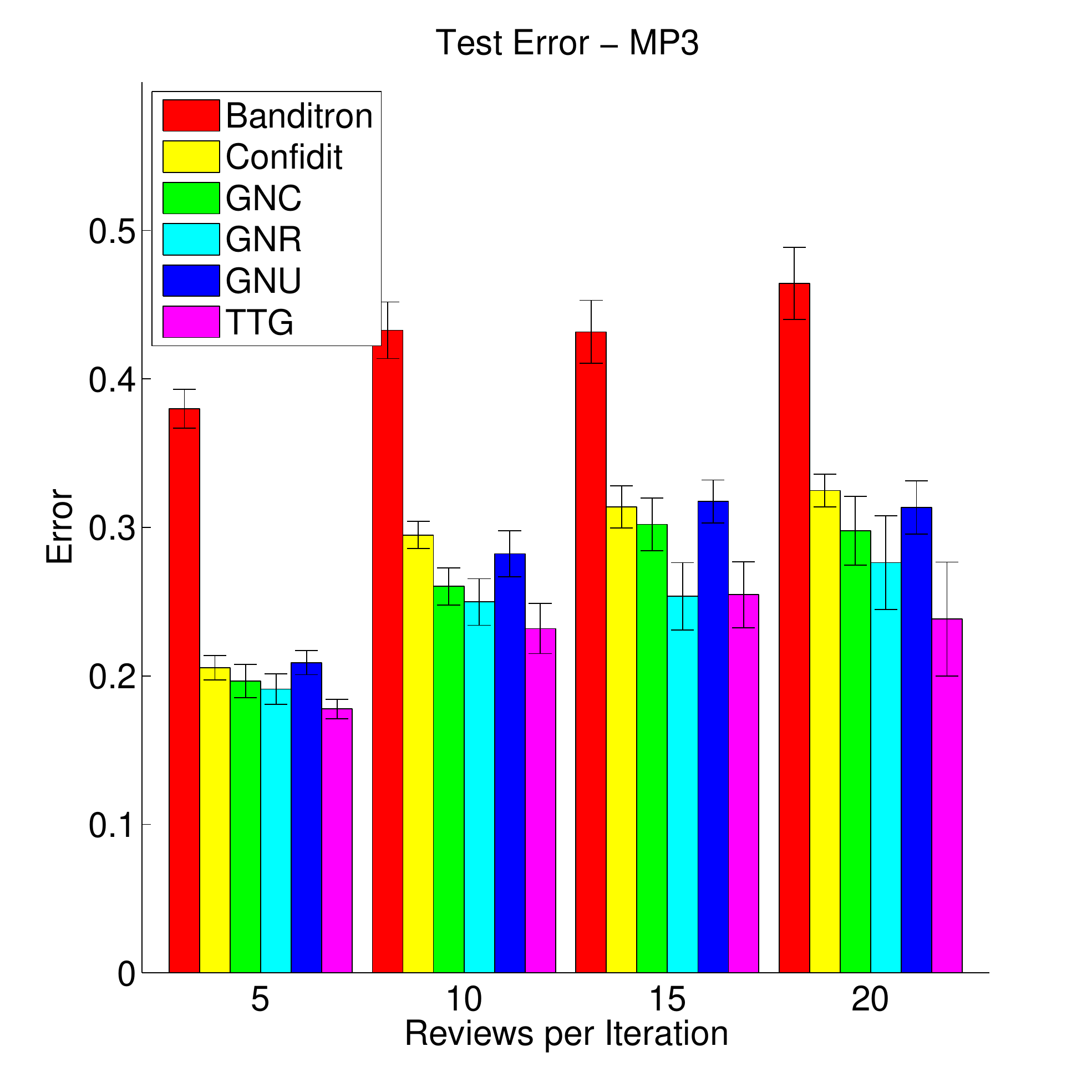}
  \includegraphics[width=0.3\textwidth]{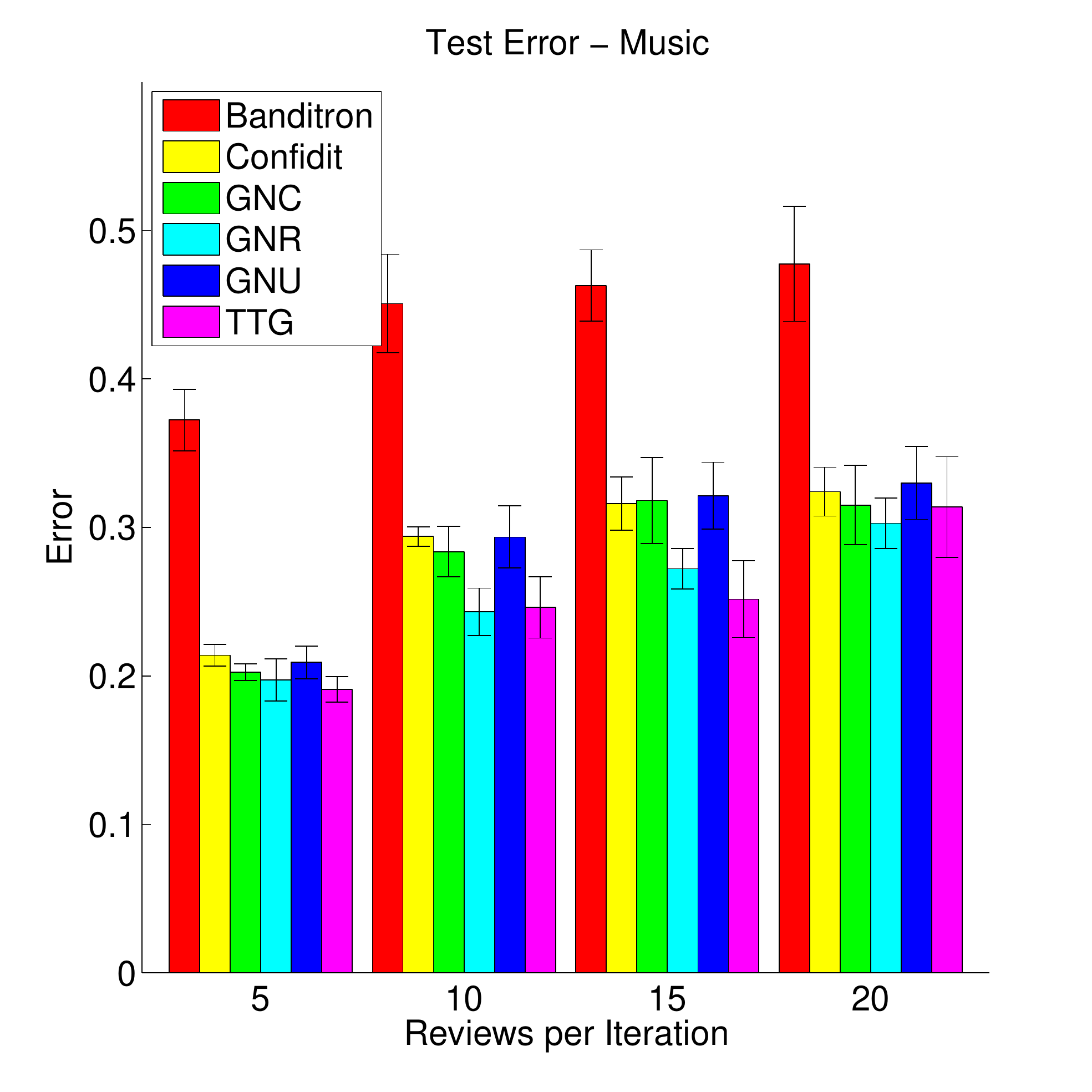}
  \includegraphics[width=0.3\textwidth]{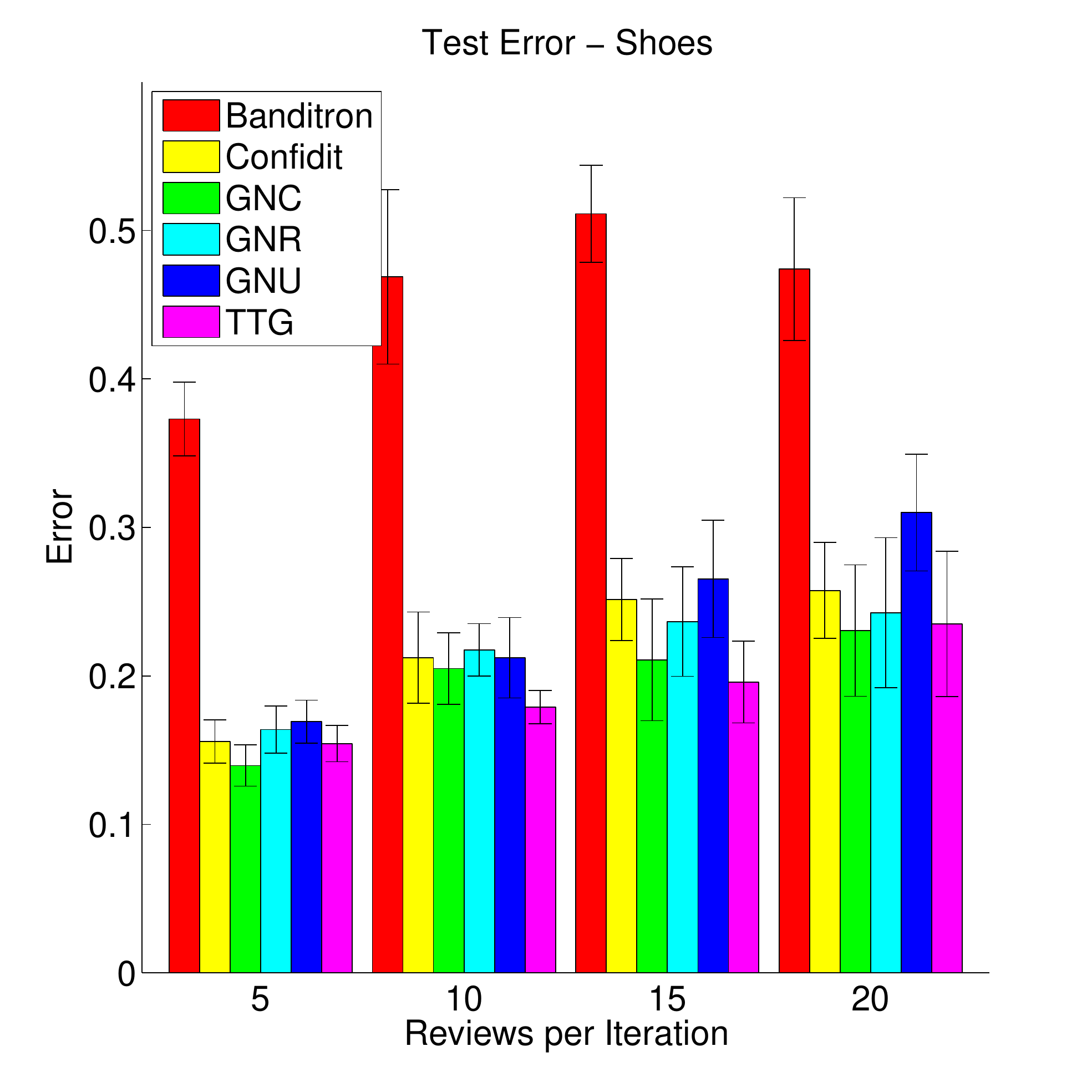}
  \includegraphics[width=0.3\textwidth]{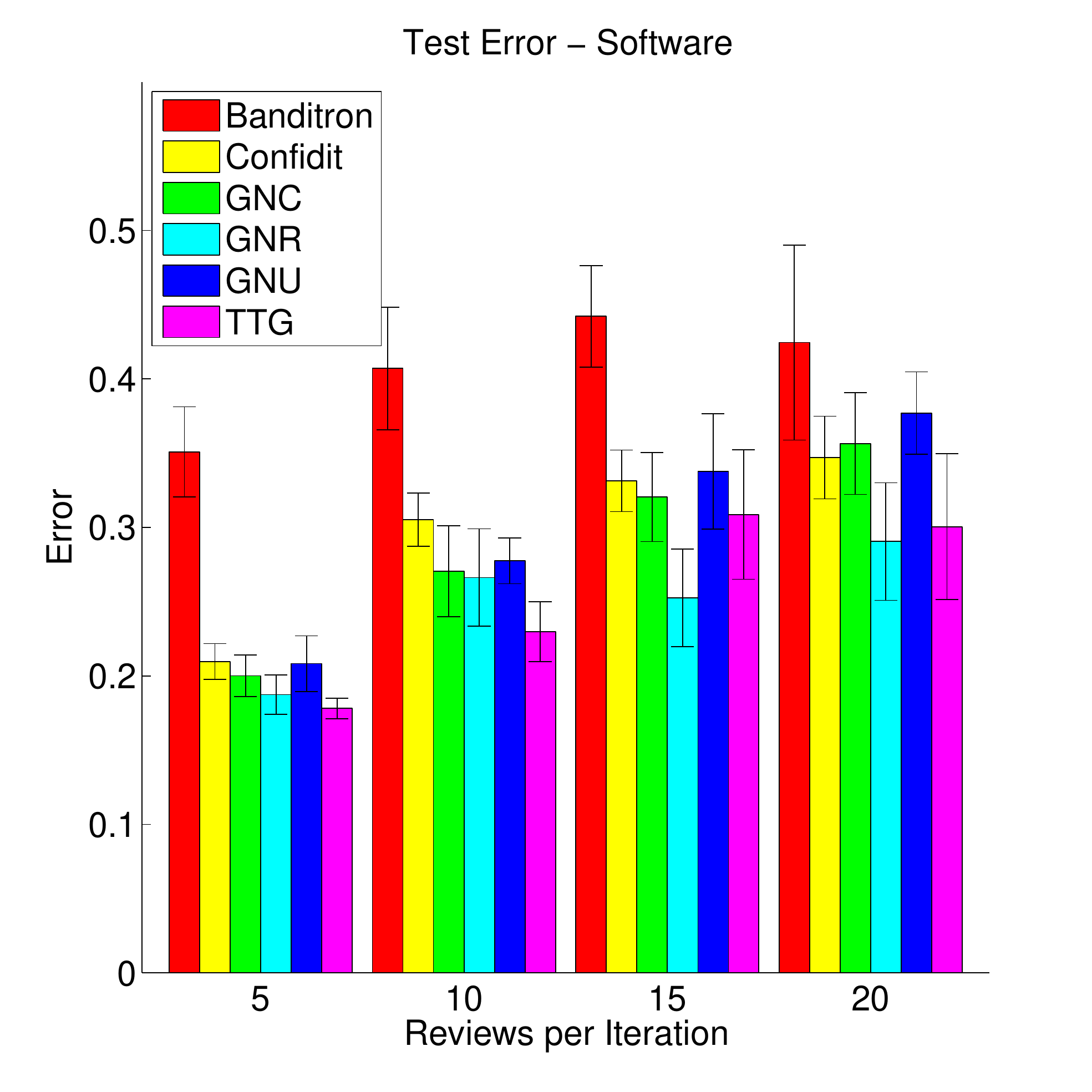}
  \includegraphics[width=0.3\textwidth]{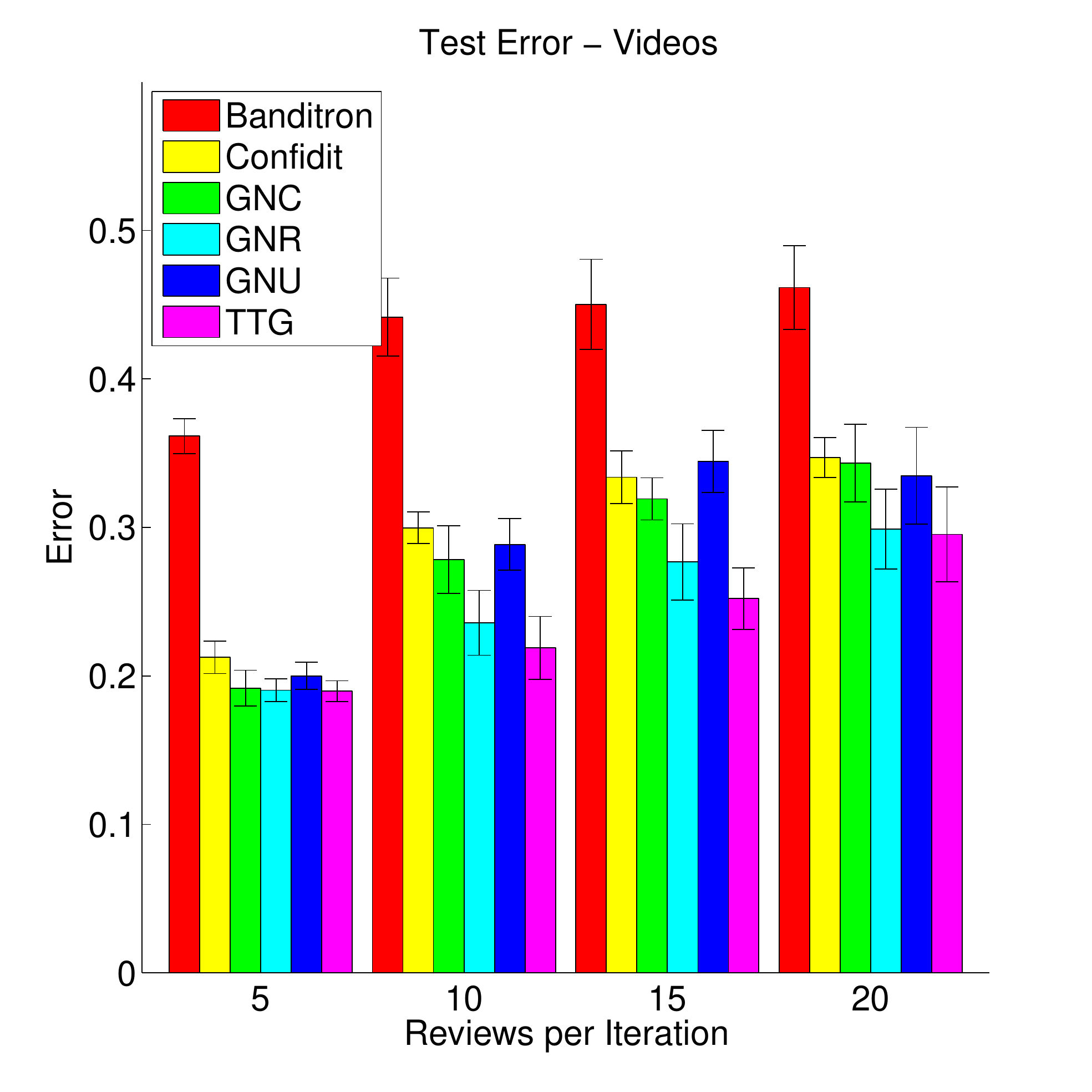}
 \caption{Average test error over Baby Products, Beauty, Clothing, Electronics, Health, Home, Movies \& TV, MP3, Music, Shoes, Software and Videos
   domains.}
 \label{fig_amazon_test2}
 \end{figure*}

\begin{figure*}[ht]
 \centering
  \includegraphics[width=0.32\textwidth]{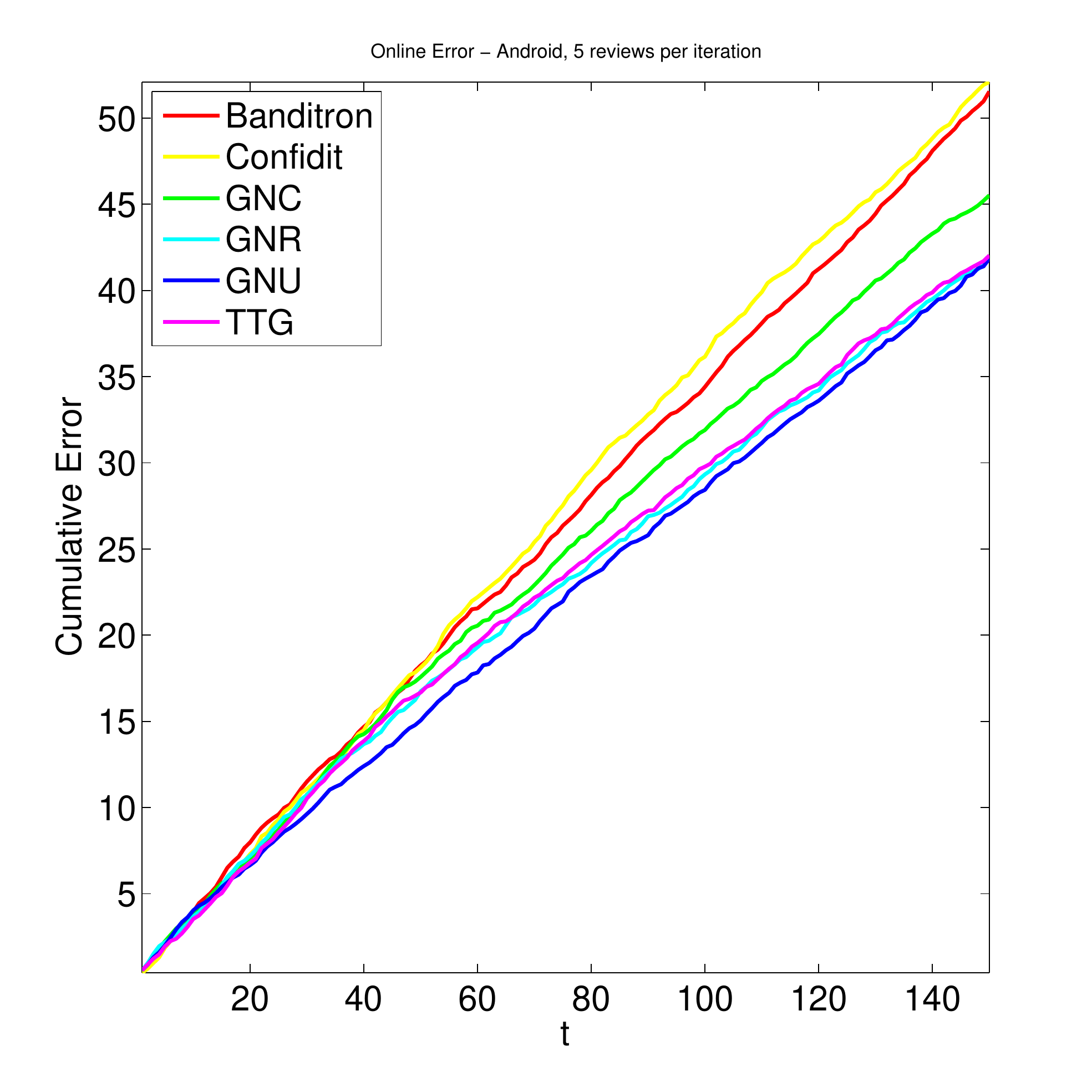}
  \includegraphics[width=0.32\textwidth]{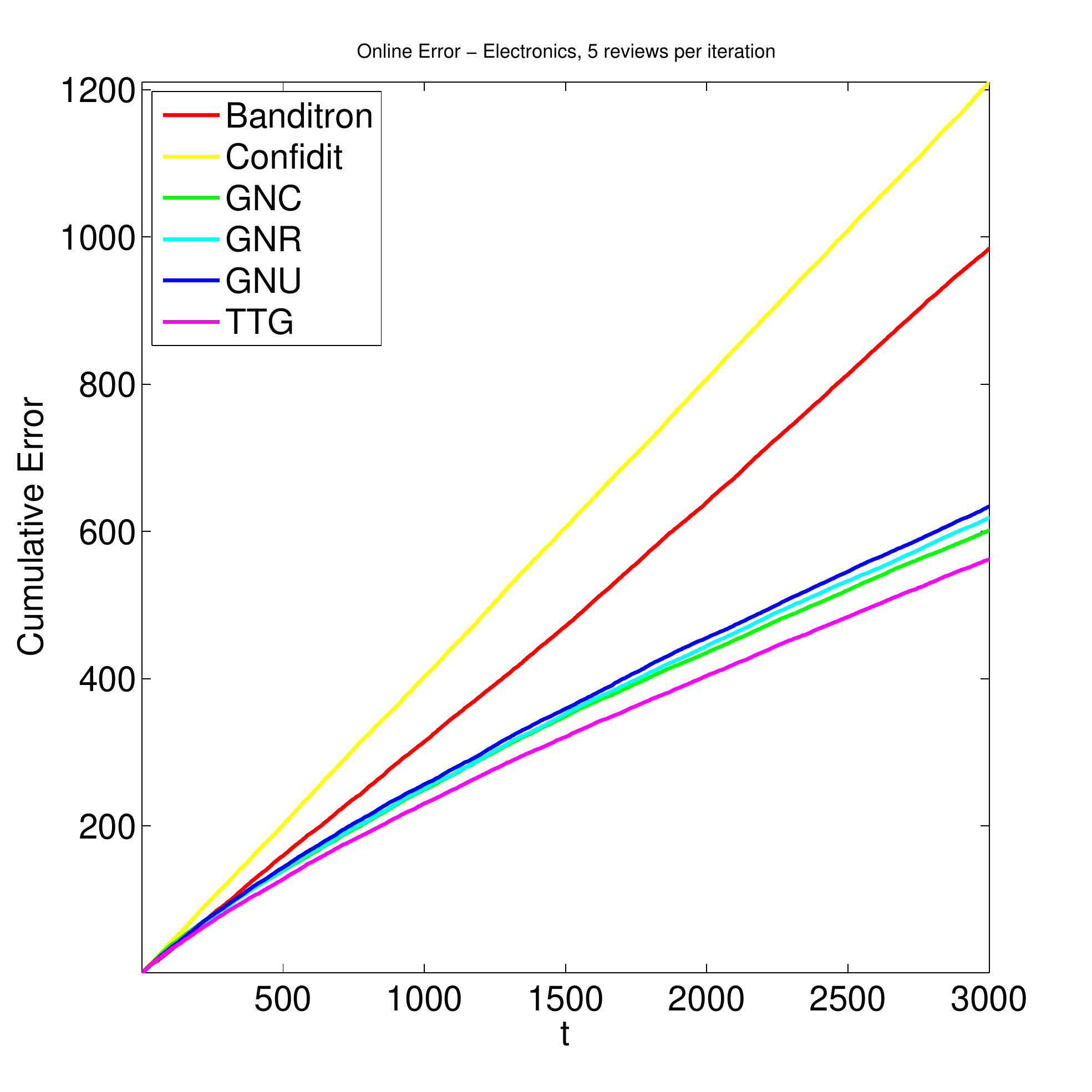}
  \includegraphics[width=0.32\textwidth]{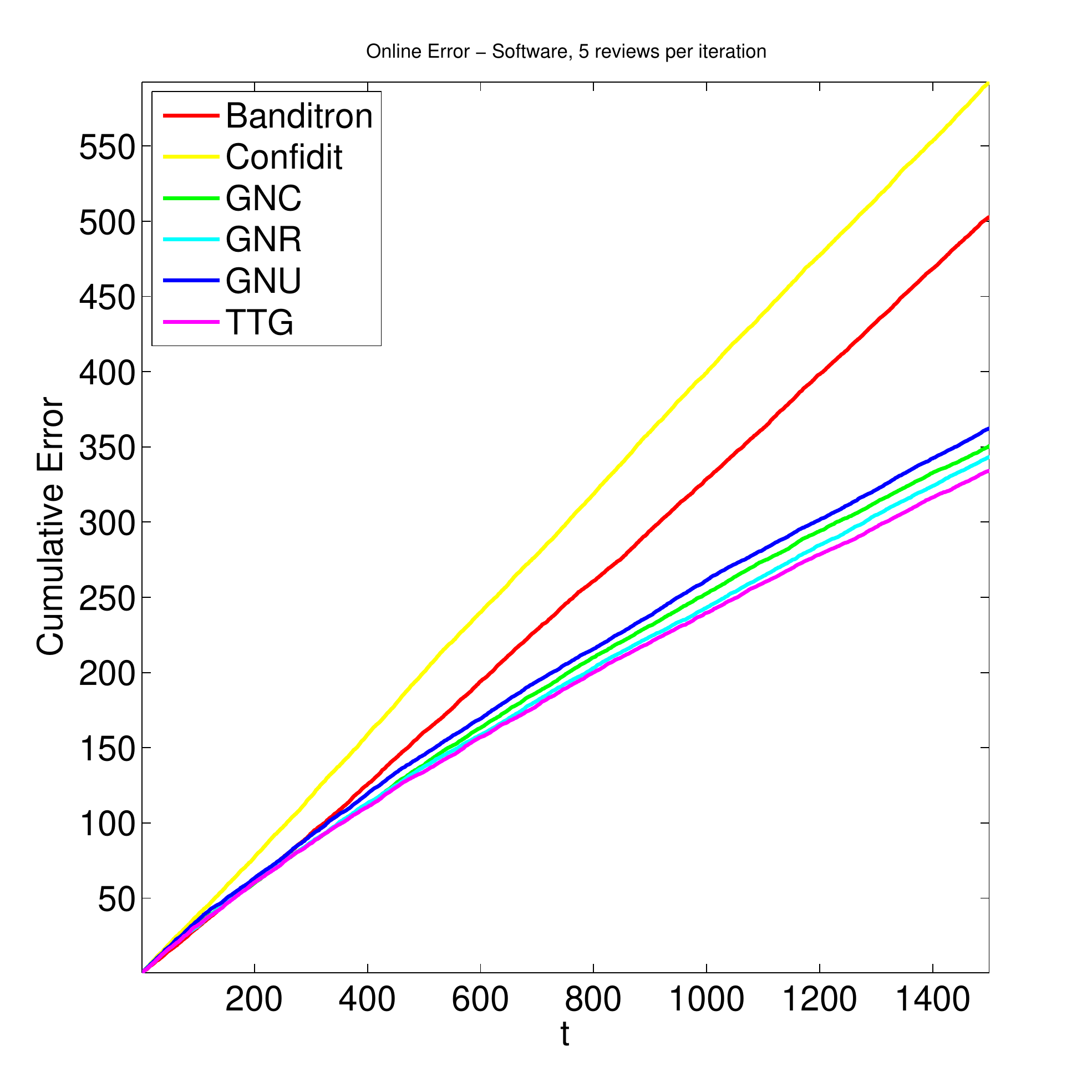}
  \includegraphics[width=0.32\textwidth]{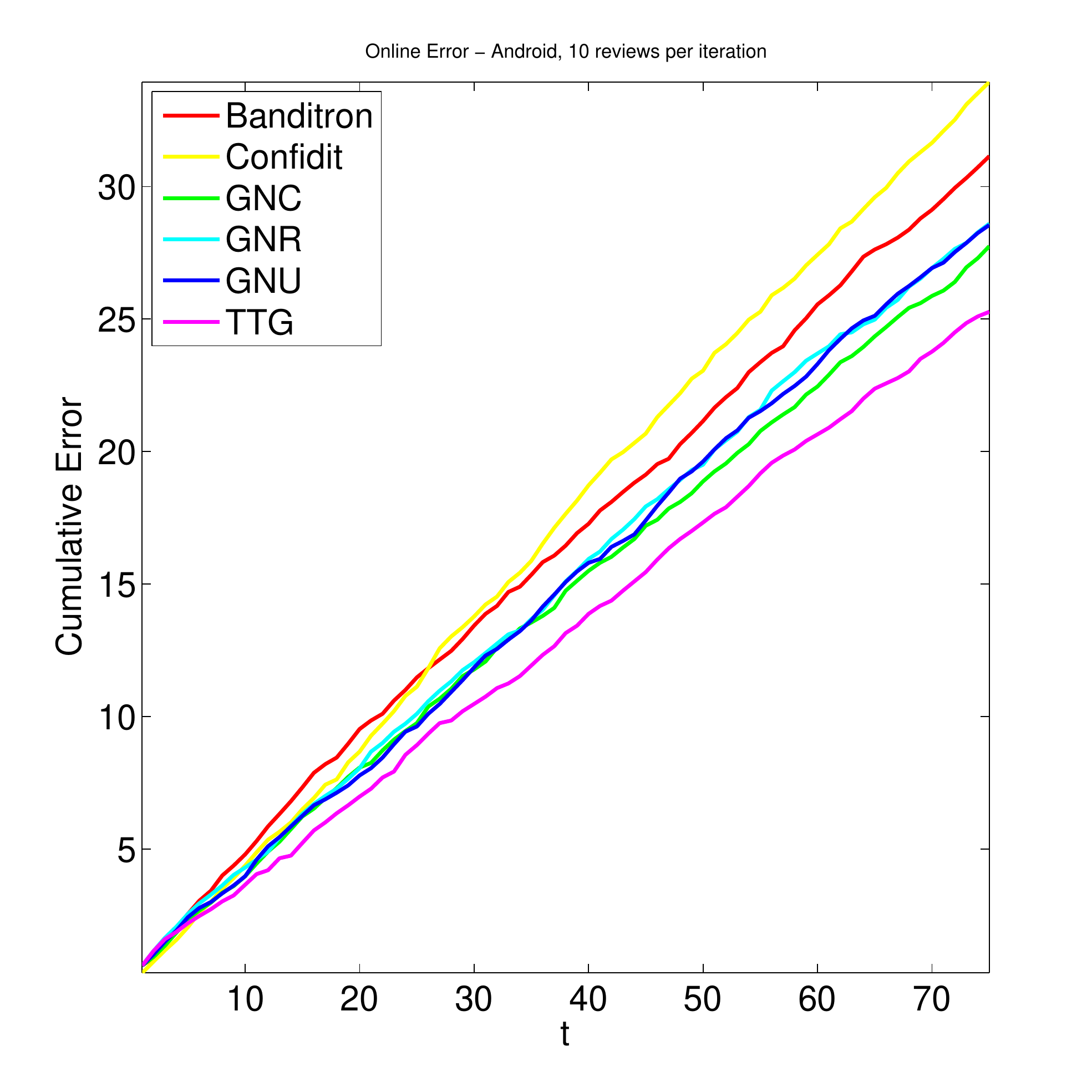}
  \includegraphics[width=0.32\textwidth]{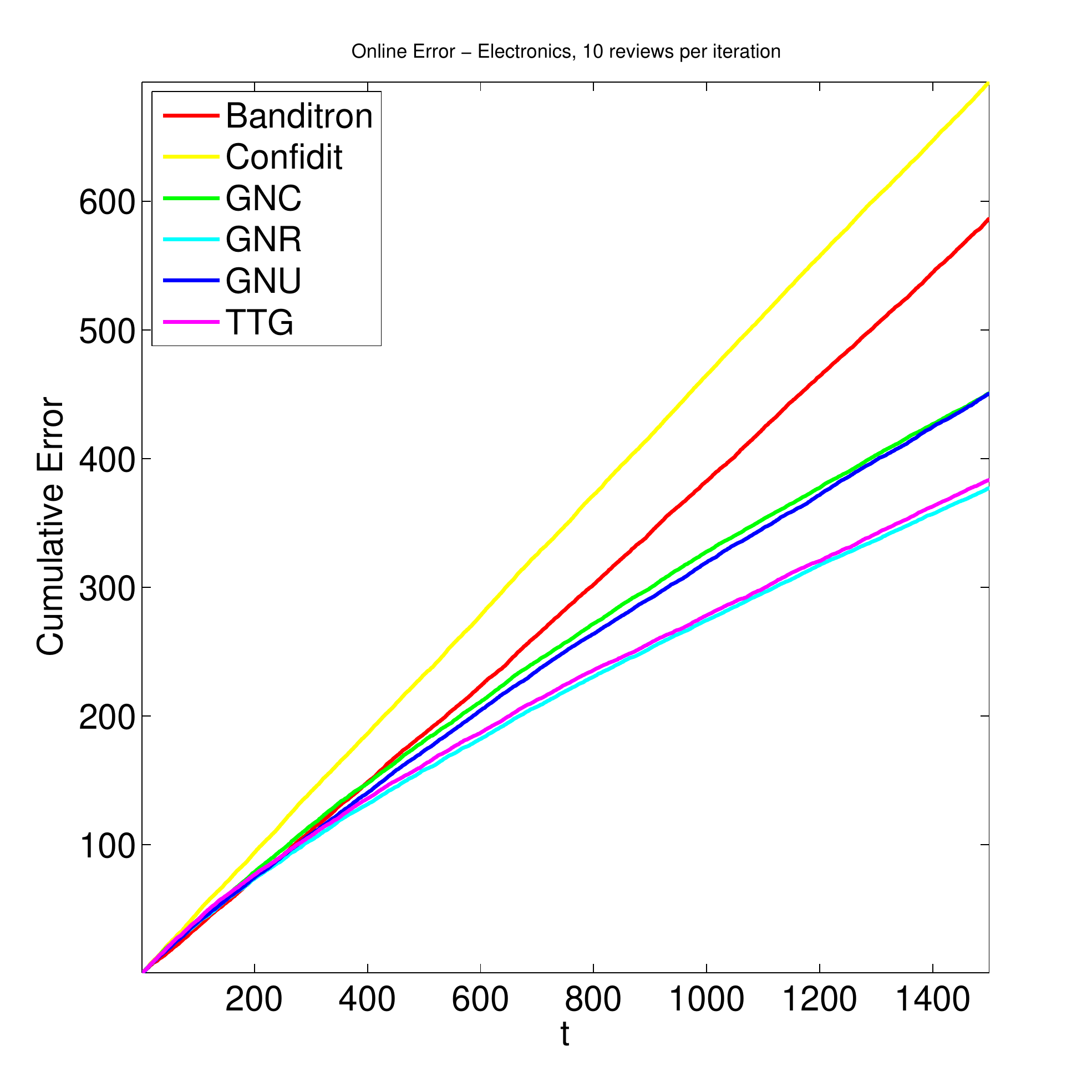}
  \includegraphics[width=0.32\textwidth]{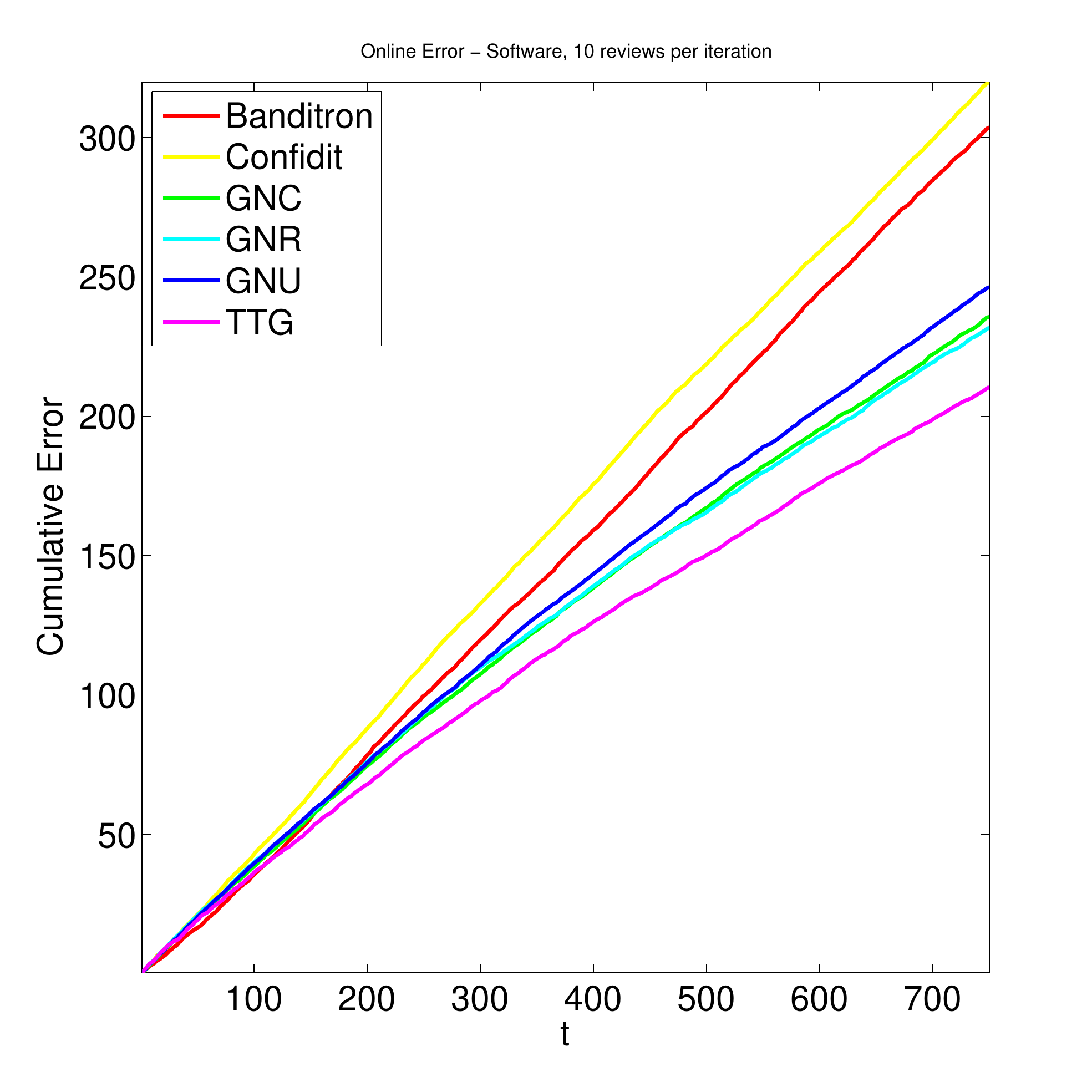}
  \includegraphics[width=0.32\textwidth]{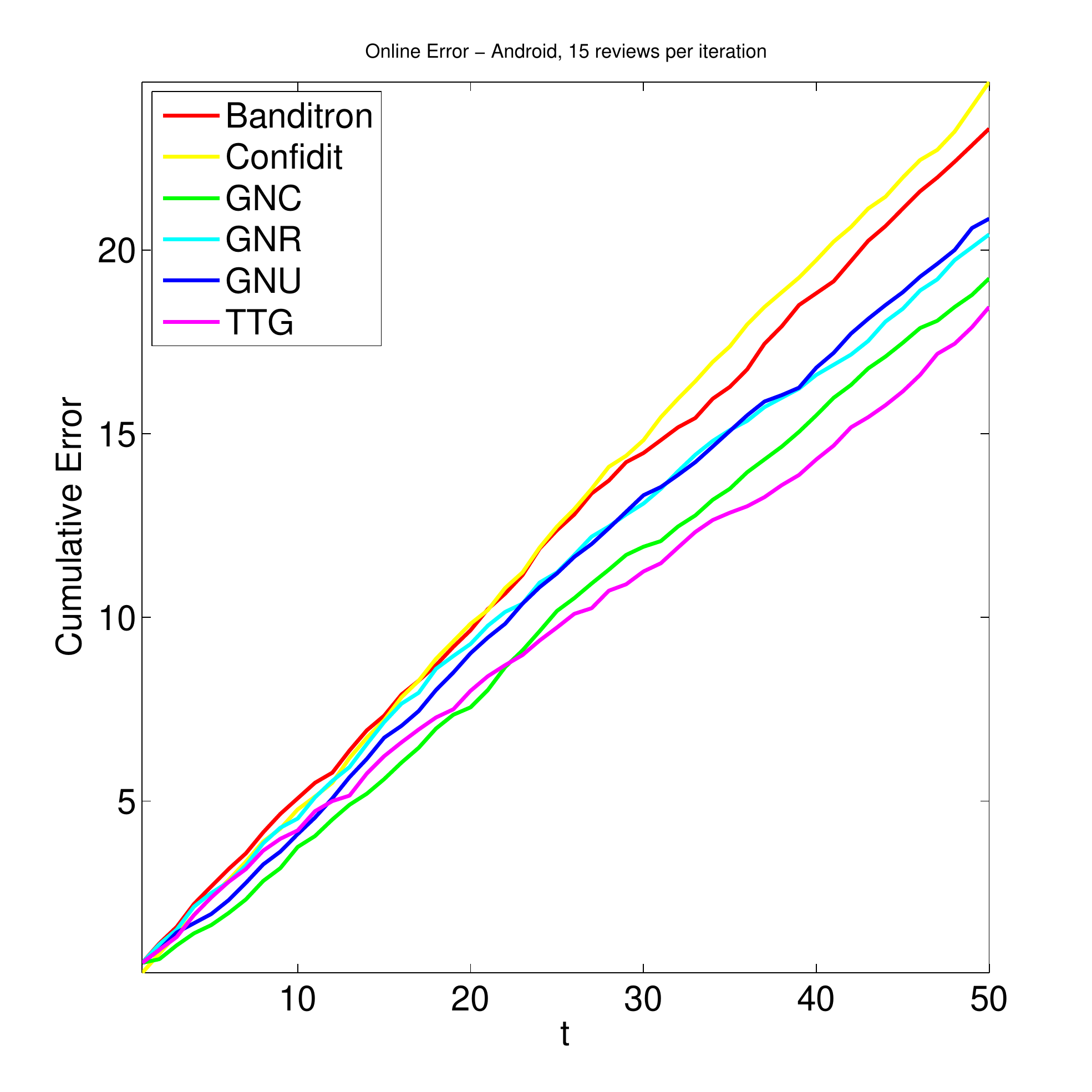}
  \includegraphics[width=0.32\textwidth]{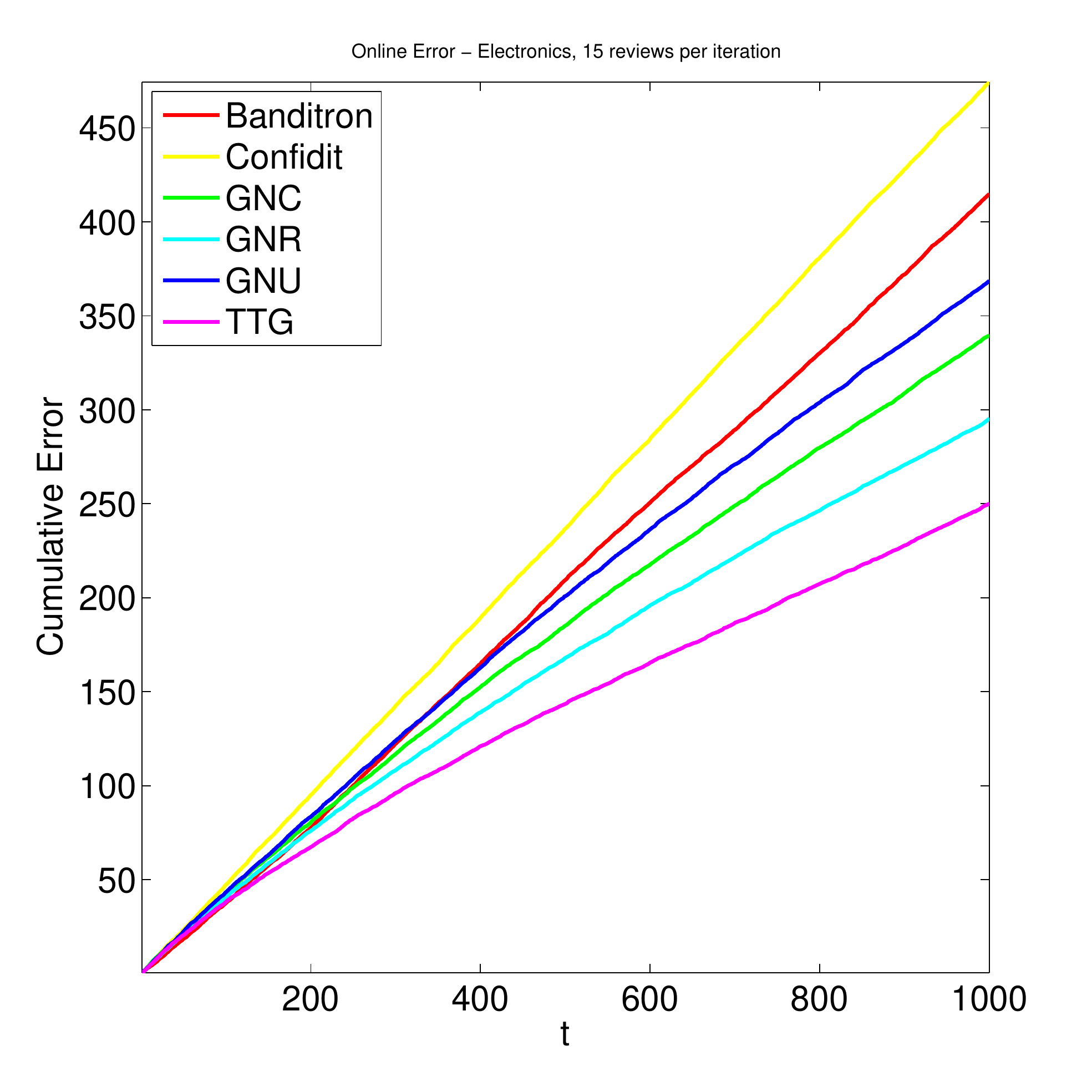}
  \includegraphics[width=0.32\textwidth]{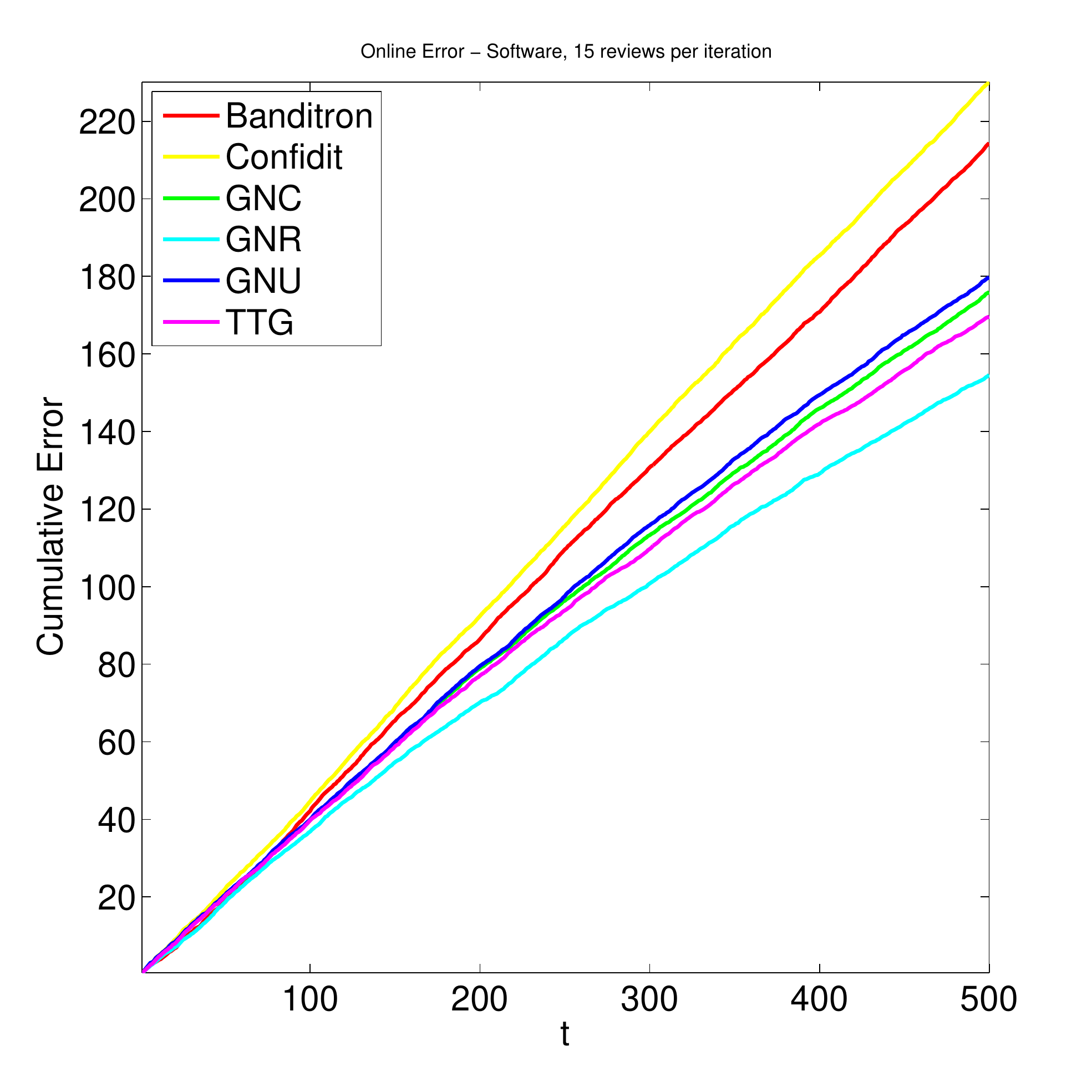}
  \includegraphics[width=0.32\textwidth]{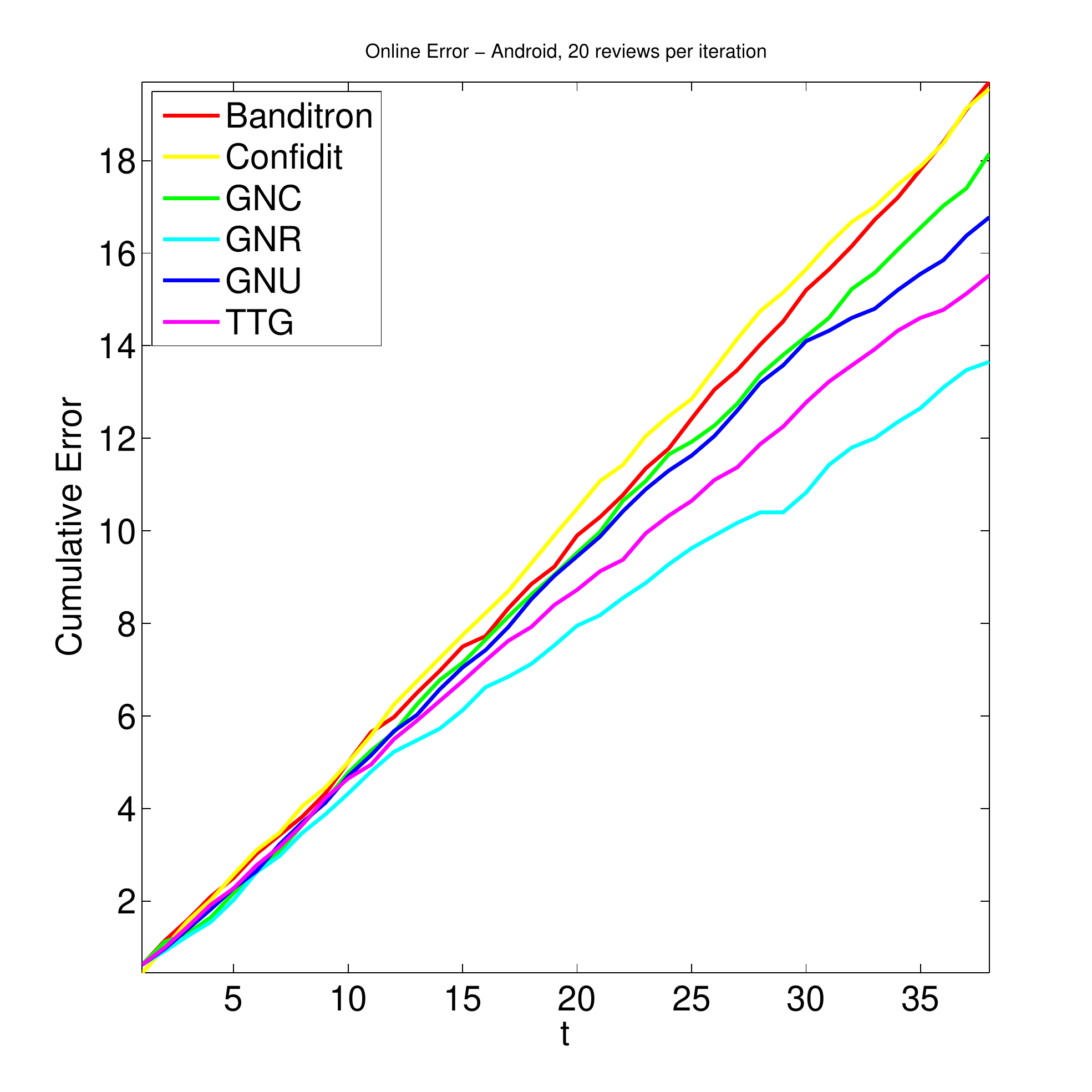}
  \includegraphics[width=0.32\textwidth]{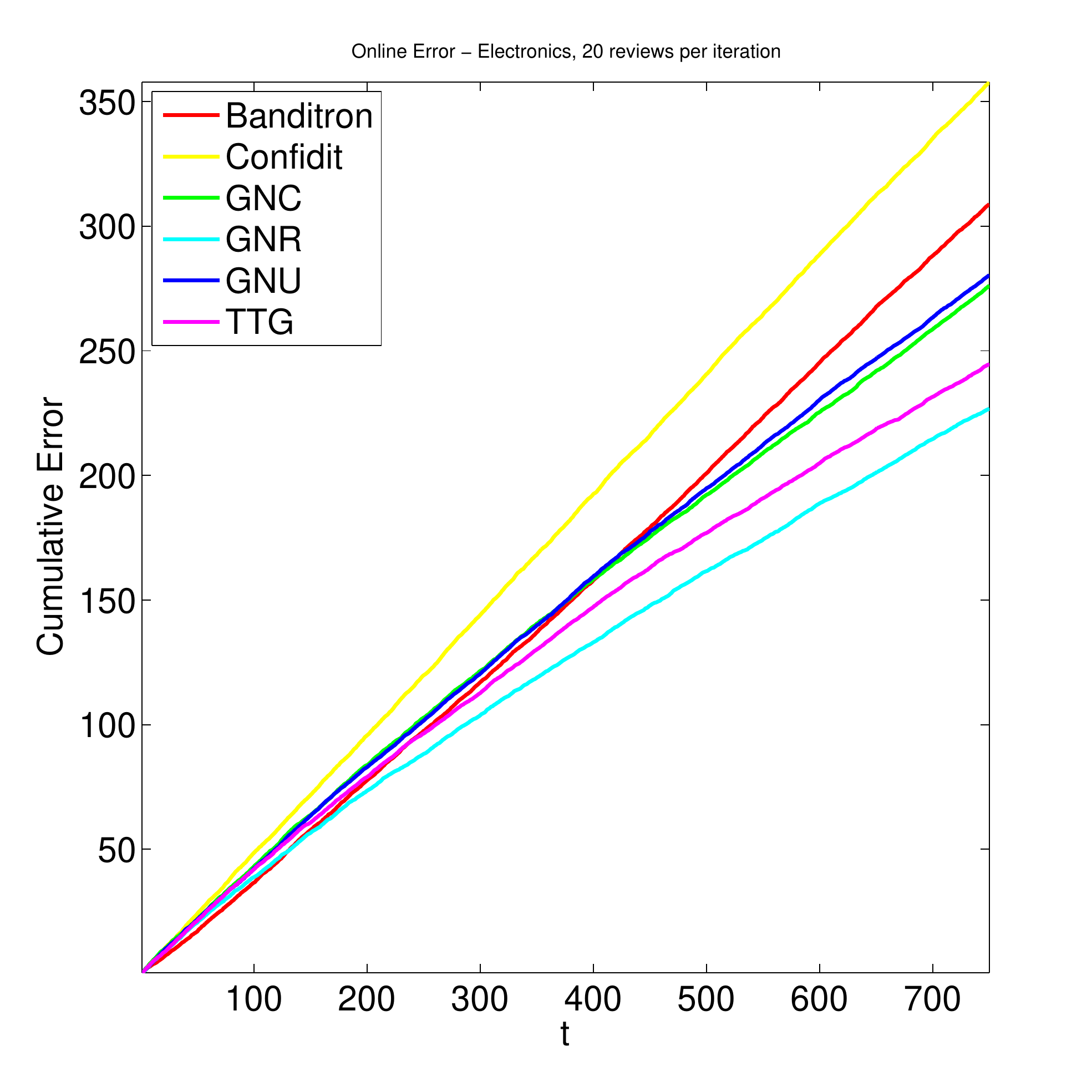}
  \includegraphics[width=0.32\textwidth]{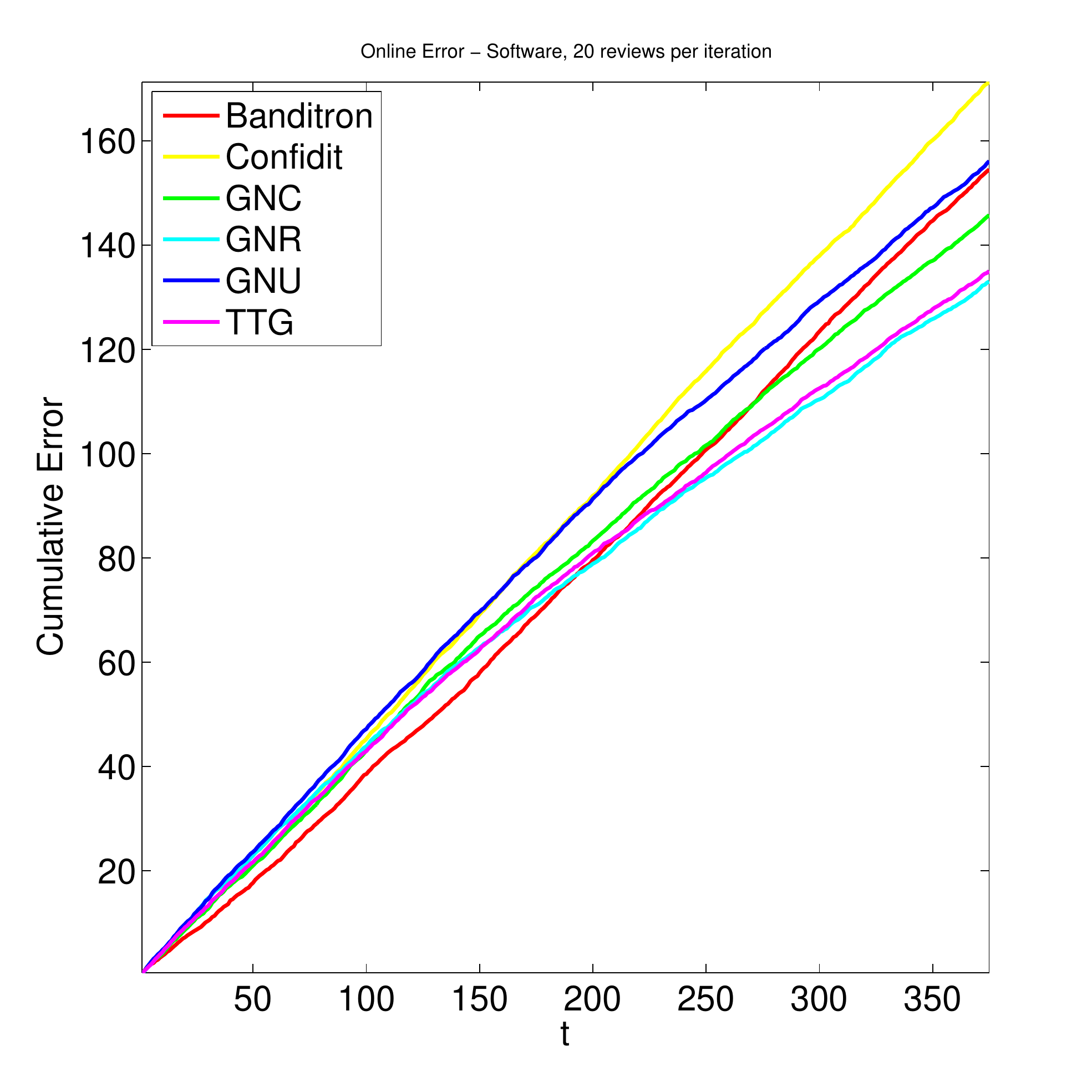}
 \caption{Cumulative online regret over Android, Electronics and Software
   domains.}
 \label{fig_amazon_train}
 \end{figure*}

\end{document}